%% file: main.tex
\documentclass[sigconf,edbt]{acmart-edbt2025}

\def\BibTeX{{\rm B\kern-.05em{\sc i\kern-.025em b}\kern-.08em
    T\kern-.1667em\lower.7ex\hbox{E}\kern-.125emX}}

\usepackage{booktabs} 
\input{preamble}

\setcopyright{rightsretained}

\acmDOI{}

\acmISBN{978-3-98318-097-4}

\acmConference[EDBT 2025]{28th International Conference on Extending Database Technology (EDBT)}{25th March-28th March, 2025}{Barcelona, Spain}
\acmYear{2025}

\begin{document}


\title{[Experiments \& Analysis] Evaluating the Feasibility of Sampling-Based Techniques for Training Multilayer Perceptrons}

\author{Sana Ebrahimi}
\affiliation{%
    \institution{University of Illinois Chicago}
    \department{Department of Computer Science}
    \city{Chicago}
    \state{Illinois}
    \country{USA}
    }
\email{sebrah7@uic.edu}

\author{Rishi Advani}
\orcid{0000-0002-5522-0401}
\affiliation{%
    \institution{University of Illinois Chicago}
    \department{Department of Computer Science}
    \city{Chicago}
    \state{Illinois}
    \country{USA}
    }
\email{radvani2@uic.edu}

\author{Abolfazl Asudeh}
\orcid{0000-0002-5251-6186}
\affiliation{%
    \institution{University of Illinois Chicago}
    \department{Department of Computer Science}
    \city{Chicago}
    \state{Illinois}
    \country{USA}
    }
\email{asudeh@uic.edu}

\input{abstract}



\maketitle

\input{intro}
\input{benefits}
\input{related}
\input{prelim}

\input{cone}
\input{ctwo}
\input{analysis}
\input{exp_setup}
\input{results}

\input{disc}

\begin{acks}
    This work was supported in part by the National Science Foundation, Grant No. 2107290, and the UIC University Fellowship.
    The authors would like to thank the anonymous reviewers and the meta-reviewer for their invaluable feedback.
\end{acks}

\bibliographystyle{ACM-Reference-Format}
\bibliography{ref-src,refs}


\end{document}

%% file: preamble.tex
\usepackage[T1]{fontenc}
\usepackage{csquotes}  
\usepackage{microtype}
\usepackage{fnpct}  

\graphicspath{{./figs/}}

\usepackage{multirow}
\usepackage{subcaption}
\usepackage{wrapfig}

\usepackage{tabularx}

\usepackage{mathtools}

\allowdisplaybreaks[1]  

\newcommand*{\Reals}{\mathbb{R}}

\DeclareMathOperator*{\argmax}{arg\,max}
\DeclareMathOperator*{\argmin}{arg\,min}
\DeclareMathOperator{\Sim}{sim}

\DeclarePairedDelimiterX\set[1]\lbrace\rbrace{\,#1\,}

\usepackage{amsthm}
\theoremstyle{plain}

\usepackage{tikz}
\usetikzlibrary{matrix, positioning, fit, decorations.pathreplacing}

\usepackage{enumitem}

\newcommand*{\cOne}{Sampling from Current Layer}
\newcommand*{\cTwo}{Sampling from Previous Layer}
\newcommand*{\cTwoShort}{Sampling from Prev.\ Layer}

\newcommand*{\mbatch}{\textsubscript{M}}  
\newcommand*{\sgd}{\textsubscript{S}}  

\newcommand*{\lsh}{\textsc{ALSH-approx}}

\newcommand*{\aprox}{\textsc{MC-approx}}
\newcommand*{\maprox}{\textsc{MC-approx\mbatch}}
\newcommand*{\saprox}{\textsc{MC-approx\sgd}}

\newcommand*{\reg}{\textsc{Standard}}
\newcommand*{\mreg}{\textsc{Standard\mbatch}}
\newcommand*{\sreg}{\textsc{Standard\sgd}}

\newcommand*{\dropout}{\textsc{Dropout}}
\newcommand*{\mdropout}{\textsc{Dropout\mbatch}}
\newcommand*{\sdropout}{\textsc{Dropout\sgd}}

\newcommand*{\addropout}{\textsc{Adaptive-Dropout}}
\newcommand*{\maddropout}{\textsc{Adaptive-Dropout\mbatch}}
\newcommand*{\saddropout}{\textsc{Adaptive-Dropout\sgd}}

\newcommand*{\actv}{\uparrow}

\usepackage{ellipsis}  

\newcommand{\stitle}[1]{\vspace{0.5mm}\noindent{\textbf{#1}}.}

\newcommand{\tightsection}[1]{\section{#1}}
\newcommand{\tightsubsection}[1]{\subsection{#1}}



%% file: abstract.tex
\begin{abstract}
The training process of neural networks is known to be time-consuming, and having a deep architecture only aggravates the issue. This process consists mostly of matrix operations, among which matrix multiplication is the bottleneck. Several sampling-based techniques have been proposed for speeding up the training time of deep neural networks by approximating the matrix products. These techniques fall under two categories: (i) sampling a subset of nodes in every hidden layer as active at every iteration and (ii) sampling a subset of nodes from the previous layer to approximate the current layer's activations using the edges from the sampled nodes. In both cases, the matrix products are computed using only the selected samples. In this paper, we evaluate the feasibility of these approaches on CPU machines with limited computational resources. Making a connection between the two research directions as special cases of approximating matrix multiplications in the context of neural networks, we provide a negative theoretical analysis that shows feedforward approximation is an obstacle against scalability. We conduct comprehensive experimental evaluations that demonstrate the most pressing challenges and limitations associated with the studied approaches. We observe that the hashing-based node selection method is not scalable to a large number of layers, confirming our theoretical analysis. Finally, we identify directions for future research.
\end{abstract}

%% file: intro.tex
\section{Introduction}\label{sec:intro}

The database community has played a significant role in tackling the complexities of big data, developing advanced technologies for addressing scalability challenges in very large data.
Many such technologies have been used to address big-data challenges across various domains. For example, \citet{asudeh2021scalable,asudeh2018leveraging} leverage sampling-based similarity joins~\citep{surajit} and selectivity estimation queries~\citep{haas,hash1} for signal reconstruction at scale.
Among such examples, data management for machine learning (ML)~\cite{kumar2017data,zhou2020database,li2021ai} shines as an interdisciplinary domain that led to many key advances, including SystemML~\citep{ghoting2011systemml,boehm2016systemml,boehm2014hybrid}, join optimization for feature selection~\citep{kumar2016join}, federated learning~\citep{mcmahan2017federated}, ML model materialization~\citep{zhang2016materialization,hasani2018efficient}, etc.
Following the theme of \emph{Extended Database Technologies for Machine Learning},
this paper studies the extension of sampling-based techniques for efficient training of deep neural networks (DNNs) on CPU machines with limited resources.

One of the crucial factors in developing DNN models with high performance is the network architecture. The task and dataset at hand determine the most appropriate architecture to use. In addition to a large number of layers, DNNs often have a high number of nodes per layer. While large models can better generalize, training them can be computationally expensive, requiring extensive amounts of data and powerful hardware, including expensive GPUs. On the other hand, \textcquote{AIOnCpu}{the ubiquity of CPUs provides a workaround to the GPU’s dominance}, motivating \textquote[Shrivastava as quoted by \citet{AIOnCpu}]{democratiz\textins*{ing} AI with CPUs}. Nevertheless, limited resources on personal computers with ordinary CPUs or mobile devices leads to difficulties in training DNNs to a sufficient level of accuracy. DNNs need to compute ``activation values'' for every layer in a forward pass and calculate gradients to update weights in the backpropagation. This requires performing expensive matrix multiplications that make the training process inefficient. Furthermore, large matrices often do not fit in the cache, and storing them in main memory necessitates constant communication between the processor and memory, which is even more time consuming.

In this work, we explore the scalability of two directions in sampling-based approaches based on locality-sensitive hashing (LSH)~\citep{LSHNN2,LSHNN1} and Monte-carlo (MC) estimations~\citep{montecarlo} for efficient training of DNNs, which can be applied on memory- and computa\-tion-constrained devices. LSH was introduced in VLDB’99~\cite{LSHNN2} for near-neighbor queries. Since then, there have been consistent contributions to this line of research, including recent publications~\citep{aumuller2022implementing,tian2023db,wei2023cryptographically}. Similarly, sampling-based Monte-carlo estimations and inner-product estimation (and search) have been two core techniques for approximate query processing~\citep{ahle2016complexity,park2018verdictdb} and Vector-DBs~\citep{pan2023survey}. We hope that evaluating sampling-based techniques for efficient training of DNNs in this paper inspires (DB) researchers to address the open problems identified.

Our contributions can be summarized as follows.
\begin{itemize}[leftmargin=*]
    \item We make a connection between two separate sampling-based research directions for training DNNs by showing that both techniques can be viewed as special cases of matrix approximation, where one samples rows of the weight matrix while the other sample its columns. To the best of our knowledge, there is no previous work in the literature to make this observation.
    \item After careful exploration of different techniques, we provide negative theoretical results that show estimation errors during the feedforward step propagate across layers. In particular, for the model of \cite{ScaleableLSH}, we prove that estimation error increases \emph{exponentially} with the number of hidden layers. 
    \item We provide extensive experiments using five training approaches and six benchmark datasets to evaluate the scalability of sampling-based approaches. Our experimental results confirm our theoretical analysis that feedforward approximation is an obstacle against scalability. In addition to other findings, our experiments reveal that while the model of \cite{FasterNN} is scalable for mini-batch gradient decent when the batch size is relatively large, there is a research gap when it comes to designing scalable sampling-based approaches for stochastic gradient decent.
\end{itemize}

The rest of our paper is organized as follows. 
We first discuss some of the potential benefits of training DNNs on CPU machines in \S\ref{sec:pros-cons}, followed by related work in \S\ref{sec:related}.

We define the problem formulation and provide a taxonomy of sampling-based approaches for efficient training in \S\ref{sec:prelim}. We discuss two of these approaches in further detail in \S\ref{sec:cone} and \S\ref{sec:ctwo}. We present our theoretical analysis in \S\ref{sec:analysis}. 

We discuss the extensions to convolutional neural networks in our technical report~\cite{techrep}.

Experiment details and takeaways are discussed in \S\ref{sec:exp_setup}--\ref{sec:discuss}, and we offer concluding remarks in \S\ref{sec:conclusion}. 

%% file: benefits.tex
\section{Benefits of Training Neural Networks on CPUs}\label{sec:pros-cons}
The pursuit of advancing DNN training on CPUs unveils a compelling avenue replete with practical advantages.
Below, we briefly explain some of these potential benefits:

\stitle{Abundance of CPU Machines}
CPU-equipped personal computing devices, including PCs and smartphones, enjoy widespread availability and accessibility among a vast segment of the population. 
Remarkably, the computational potential of these devices often remains underutilized. 
Leveraging such resources for DNN training introduces the opportunity to conduct this computational-intensive task at no additional hardware cost for personal endeavors. 
Furthermore, while individual devices possess limited capacity, their collective potential can effectively address a multitude of moderate-sized artificial intelligence (AI) challenges. Recognizing this collective capability, recent endeavors have emerged to design client-side AI frameworks, exemplified by JavaScript packages like Tensorflow.js, facilitating machine learning on the client-side. Advancements in DNN training on CPU machines directly benefit these platforms.

\stitle{Independence from Backend Servers}
Personalized AI necessitates the training, or at the very least, fine-tuning of machine learning models with user-specific data. Opting for DNN training on CPU machines renders this process independent of backend GPU servers. Instead of transmitting data to the server, each personal device can locally fine-tune the models using its own data. This approach instantly confers several additional advantages:
\begin{description}
    \item[Privacy] By refraining from transmitting data to a server, concerns regarding data privacy are substantially alleviated.
    \item[Reduced Backend Computation] The computational burden is shifted to the client side, at no expense to the server.
    \item[Offline Availability] By localizing computations, the need for communication with a server is eliminated. This is especially useful for users with limited/unreliable internet access.
\end{description}

\stitle{Democratizing DNN Training}
GPU-equipped machines, while gradually becoming more affordable, still pose a considerable financial barrier. These costs manifest in the form of GPU access or enterprise APIs, especially in the context of large models like ChatGPT. Consequently, such resources remain inaccessible to a significant portion of the population. The facilitation of DNN training on CPU machines effectively dismantles this accessibility barrier.

Incorporating these considerations into the discourse of DNN training on CPU machines not only enriches the academic discussion but also underscores the profound implications of this research direction in addressing pressing real-world challenges and democratizing AI accessibility. It is evident that significant research efforts have been judiciously directed towards the principal trajectory of DNN training on GPU systems. Conversely, the avenue of training DNNs on CPU machines remains relatively under-explored within the research landscape. 

%% file: related.tex
\section{Related Work}\label{sec:related}
The increasing importance of DNN applications opened the door to a variety of challenges associated with training these models. While there are numerous works on techniques for scaling DNNs, many of them have expensive hardware requirements and use GPUs to accelerate training \citep{GPUMM}. Unlike GPUs, CPUs are available on any device, so optimizing training performance on CPUs is beneficial. There have been studies in which distributed, concurrent, or parallel programming on CPUs has been used to accelerate training \citep{ScaleableLSH, DistLargeNN, OptTrainCPU, deep-pruning-limitedresources, mobile-DNN, MM-CPU}, but these methods are not always applicable due to variation in hardware requirements. Thus, algorithms focused on algorithmic optimization of feedforward and backpropagation are essential. Usually, methods with little to no special hardware requirements are preferred. Several algorithms apply a variety of sampling-based \citep{ScaleableLSH, ImpDropDNN, AdaptivDrop, FasterNN, statesparsity, unifiedCompression} or non-sampling-based approximations \citep{WTA, layerwiseImportance, non-sampling-opt, deep-pruning-limitedresources} to improve training for DNNs.

Several studies have shown that one way to scale up DNNs is to simplify the matrix-matrix or vector-matrix operations involved \citep{Dropout, AdaptivDrop, FastDrop, FasterNN, ScaleableLSH}. The complexity of matrix multiplication dominates the computational complexity of training a neural network. The multiplication of large matrices is known to be the main bottleneck in training DNNs. Often, we try to sparsify the matrices, which can minimize the communication between the memory and the processors \citep{sparse-mm-opt}. Pruning the network and limiting the calculations in both directions to a subset of nodes per layer is one solution to train DNNs efficiently. This is what dropout-type algorithms suggest \citep{Dropout, AdaptivDrop, statesparsity}. Dropout-type methods either use a data-dependent sampling distribution \citep{AdaptivDrop, FastDrop, FasterNN} or a predetermined sampling probability \citep{Dropout, ScaleableLSH}. Note that these techniques are able to provide a good approximation only if used in the context of neural networks; they are not necessarily applicable to general matrix multiplication.

%% file: prelim.tex
\section{Preliminaries}\label{sec:prelim}
In this section, we describe the neural network model, feedforward step, and backpropagation in the form of matrix operations. Finally, we discuss state-of-the-art sampling-based algorithms.

\subsection{Problem Description}\label{sec:pre:defprob}
Many neural network architectures have been studied over the past decade. In this paper we focus on the standard multi-layer perceptron (MLP) model and analyze two major directions of sampling-based techniques for training neural networks.

Consider a feedforward neural network with $m_i$ inputs, $m_o$ outputs, and $\ell$ hidden layers. In general, every hidden layer $k$ contains $n_k$ hidden nodes while the nodes in the $(k-1)$-th and $k$-th layers are fully connected. Without loss of generality, for ease of explanation, we assume all hidden layers have exactly $n$ hidden nodes (Figure~\ref{fig:NN-Arch}).
For each layer $k$, we denote the vector of outputs by $a^k \in \Reals^{1 \times n}$. 
 Similarly, $W^k \in \Reals^{n \times n}$ and $b^k\in \Reals^{1 \times n}$ are the weights and the biases of layer
$k$\footnote{The first and last layer are exceptions: $W^1 \in \Reals^{m_i \times n}$ and $W^{\ell} \in \Reals^{n \times m_o}$.}, respectively.

\begin{figure}[bpt]
\small
\centering
\begin{tikzpicture}
\tikzset{node distance=8mm}
\tikzset{every node/.style={circle, draw, fill=green!15!white, inner sep=0.5mm, minimum size=6mm}}
\tikzset{blank/.style={draw=none,fill=none, inner sep=0, minimum size=3mm}}
\tikzset{every path/.style={thick, -}}


\node (n01) {$x_{1}$};
\node (n02) [below = of n01] {$x_2$};
\node[blank] (n0dots) [below = of n02, yshift=5mm] {{\huge\vdots}};
\node (n0mi) [below = of n0dots, yshift=4mm] {$x_{m_i}$};

\node[blank] [left = of n01, xshift=8mm] (i1) {};
\node[blank] [left = of n02, xshift=8mm] (i2) {};
\node[blank] [left = of n0mi, xshift=8mm] (imi) {};

\node (n11) [right = of n01,yshift=1.5mm] {$n_{1,1}$};
\node (n12) [below = of n11] {$n_{1,2}$};
\node[blank] (n1dots) [below = of n12, yshift=5mm] {\huge\vdots};
\node (n1n) [below = of n1dots, yshift=4mm] {$n_{1,n}$};

\node (n21) [right = of n11] {$n_{2,1}$};
\node (n22) [below = of n21] {$n_{2,2}$};
\node[blank] (n2dots) [below = of n22, yshift=5mm] {\huge\vdots};
\node (n2n) [below = of n2dots, yshift=4mm] {$n_{2,n}$};

\node (n31) [right = of n21, yshift=-8mm] {$n_{3,1}$};
\node[blank] (nldots) [below = of n31, yshift=5mm] {\huge\vdots};
\node (n3mo) [below = of nldots, yshift=4mm] { $n_{3,m_o}$};

\node[blank] (o1) [right = of n31, xshift=-8mm] {};
\node[blank] (omo) [right = of n3mo, xshift=-8mm] {};


\draw[->] (i1.center) -- (n01);
\draw[->] (i2.center) -- (n02);
\draw[->] (imi.center) -- (n0mi) node [blank, below=3mm] {Input};

\draw (n01) -- (n11) node [blank, midway, above] {$W^1$} node [blank, midway, below=-1mm] {\tiny $W^1_{1,1}$};
\draw (n01) -- (n12);
\draw (n01) -- (n1n);
\draw (n02) -- (n11);
\draw (n02) -- (n12);
\draw (n02) -- (n1n);
\draw (n0mi) -- (n11);
\draw (n0mi) -- (n12);
\draw (n0mi) -- (n1n) node [blank, midway, below=-2mm] {\tiny $W^1_{m_i,n}$} node [blank, below=1mm] {Layer 1};

\draw (n11) -- (n21) node [blank, midway, above] {$W^2$} node [blank, midway, below=-1mm] {\tiny $W^2_{1,1}$};
\draw (n11) -- (n22);
\draw (n11) -- (n2n);
\draw (n12) -- (n21);
\draw (n12) -- (n22);
\draw (n12) -- (n2n);
\draw (n1n) -- (n21);
\draw (n1n) -- (n22);
\draw (n1n) -- (n2n) node [blank, midway, below=-1mm] {\tiny $W^2_{n,n}$} node [blank, below=1mm] {Layer 2};

\draw (n21) -- (n31) node [blank, midway, above, xshift=1mm] {$W^3$} node [blank, midway, below=-1mm] {\tiny $W^3_{1,1}$};
\draw (n21) -- (n3mo);
\draw (n22) -- (n31);
\draw (n22) -- (n3mo);
\draw (n2n) -- (n31);
\draw (n2n) -- (n3mo) node [blank, midway, below=-2mm, xshift=2mm] {\tiny $W^3_{n,m_o}$} node [blank, below=5mm] {Output};

\draw[->] (n31) -- (o1.center);
\draw[->] (n3mo) -- (omo.center);
\end{tikzpicture}
\caption{Neural network with $\ell=3$ layers and $n$ nodes per hidden layer.}
\label{fig:NN-Arch}
\end{figure}
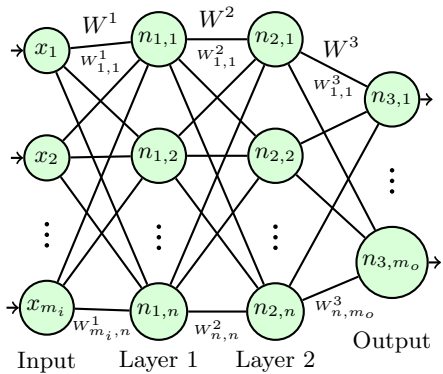

Let $f$ be the activation function (e.g., sigmoid or ReLU). With input vector $x \in \Reals^{1 \times m_i}$, the feedforward step is a chain of matrix products and activation functions that maps input vector $x$ to output vector $y$; it can be represented as follows (with $a^0 = x$):
\begin{align*}
    z^k = a^{k-1} W^k + b^k && a^k = f(z^k)
\end{align*}

In the setting described above, matrix-vector multiplication can be done in $\Theta(n^2)$ time and applying the element-wise activation function takes $\Theta(n)$ time for each layer. Thus, the entire feedforward process for the whole network is in the order of
$\Theta(\ell n^2)$.

The final aspect in the training of neural networks is backpropagation, an efficient method of computing gradients for the weights to move the network towards an optimal solution via stochastic gradient descent\footnotemark{} (SGD) \citep{DL}.
\footnotetext{While SGD uses only one data point to compute the gradients, an alternative approach is mini-batch gradient descent (MGD), where a small sample set (mini-batch) of the training set is used for estimations. Note that SGD can be viewed as a special case of MGD where the batch size is 1. Following our scope in this paper, SGD is considered for problem formulation, explaining the learning algorithm, and analysis. Nevertheless, as we shall later explain in \S\ref{sec:method:mm}, one of the evaluated approaches, \aprox{}, is based on MGD, the generalization of SGD. While SGD operations are in form of vector to matrix multiplication, MGD operations are in form of matrix (vectors of samples in the mini-batch) to matrix multiplication.}
The weight gradients for the backpropagation step can be computed recursively using Equation~\ref{eq:backprop}, where $\mathcal{L}$ is the loss function and $\odot$ is the Hadamard product.
\begin{equation}\label{eq:backprop}
\begin{aligned}
    &\delta^\ell = \nabla_{z^\ell} \mathcal{L} = f'(z^\ell) \odot \nabla_{a^\ell} \mathcal{L}
    &&\nabla_{W^k} \mathcal{L} = a^k \delta^{k+1} \\
    &\delta^k = \nabla_{z^k} \mathcal{L} = f'(z^k) \odot W \delta^{k+1} \qquad
    &&\nabla_{b^k} \mathcal{L} = \delta^{k+1}
\end{aligned}
\end{equation}
With gradient $a^k \delta_{k+1}$ and learning rate $\eta$, the weight matrix $W^{k}$ will be updated to $W^k - \eta a^k \delta_{k+1}$. The gradient computation and update operations are also in form of vector-matrix operations that take $\Theta(n^2)$ time for each layer. As a result, the backpropagation step in SGD also requires $\Theta(\ell n^2)$ time.

\begin{table}[bpt]
\small
\caption{Table of Notations}
\begin{tabularx}{\linewidth}{cX}
\toprule
{Notation} & {Description} \\
\midrule
$a^k$ & output vector of $k$-th layer \\
$z^k$ & input vector of the $k$-th layer \\
$W^k$ & weight matrix of the $k$-th layer \\
$W^k_{i,:}$ & row $i$ of $W^k$ \\
$W^k_{:,j}$ & column $j$ of $W^k$ \\
$\langle x_1, x_2 \rangle$ & inner product of $x_1$ and $x_2$ \\
\bottomrule
\end{tabularx}
\label{tab:notations}
\end{table}

\subsection{Taxonomy of Sampling-Based Techniques}\label{sec:methods}
The computation bottleneck in the training of a DNN is matrix multiplication, in form of a $(1\times n)$ to $(n \times n)$ vector-matrix product for SGD. Sampling-based approaches seek to speed up this operation by skipping a large portion of the scalar operations. SGD is a noisy algorithm by nature. As such, it is more tolerant of small amounts of noise \citep{precisionMM}, allowing for approximation. At a high level, these approaches fall in two categories, as shown below:

\begin{center}
\small
\begin{tikzpicture}[
level 1/.style={sibling distance=4cm, level distance=1.1cm}, level 2/.style={sibling distance=4cm, level distance=1.1cm}]

\node {Sampling-Based Techniques}
    child {node[align=center] {Sampling from\\ Current Layer}
        child {node[align=center] {Dropout \\ \cite{AdaptivDrop,Dropout}}}
        child {node[align=center] {\lsh{} \\ \cite{ScaleableLSH}}}
    }
    child {node[align=center] {Sampling from\\ Previous Layer}
        child {node[align=center] {\aprox{} \\ \cite{FasterNN,MonteCarloMM}}}
    };
\end{tikzpicture}
\end{center}

\stitle{\cOne{}}
The approaches in this category select a small subset of nodes in each layer during each feedforward--backpropagation step, and update the values of only those nodes. In Figure~\ref{fig:matrix}, each column $W^k_{:,j}$  corresponds to the node $n^k_{j}$ in layer $k$, while each cell $W^k_{i,j}$ in that column represents the weight of the edge from $n^{k-1}_{i}$ to $n^k_{j}$. As a result, these approaches can be viewed as \emph{selecting a small subset of the columns} of $W^k$ (e.g., the highlighted columns) and conducting the inner product only for those.

\stitle{\cTwo{}}
Instead of selecting a subset of columns and computing the exact inner-product for them, the alternative is to select all columns but compute the inner-product approximately for them by \emph{selecting a small subset of rows} of $W^k$ (e.g., highlighted rows in Figure~\ref{fig:matrix}). That is, instead of computing the sum for all $n$ scalars in an inner-product, to \emph{estimate} the sum by sampling a small number of scalars.

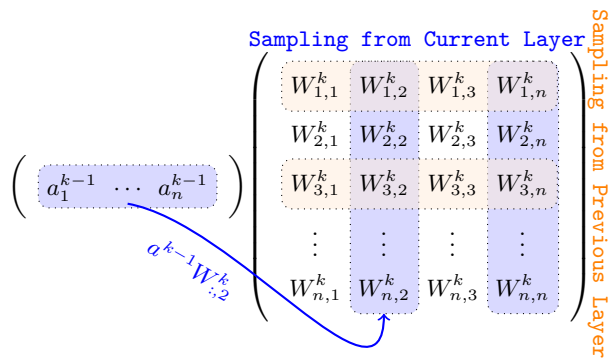
\begin{figure}[bpt]
\small
\centering
\begin{tikzpicture}
\pgfdeclarelayer{background}
\pgfdeclarelayer{foreground}
\pgfsetlayers{background,main,foreground}
\matrix (m) [matrix of math nodes,
ampersand replacement=\&,
left delimiter = (,
right delimiter = ),
nodes={minimum size=0.5em}]
{
a^{k-1}_1   \&   \cdots \& a^{k-1}_n\\
};
\matrix  [matrix of math nodes,
right =63pt of m-1-1.east,
ampersand replacement=\&,
left delimiter = (,
right delimiter = ),
nodes={minimum size=2em}] (j)
{
W^{k}_{1,1} \& W^{k}_{1,2} \& W^{k}_{1,3} \& W^{k}_{1,n} \\
W^{k}_{2,1} \& W^{k}_{2,2} \& W^{k}_{2,3} \& W^{k}_{2,n} \\
W^{k}_{3,1} \& W^{k}_{3,2} \& W^{k}_{3,3} \& W^{k}_{3,n} \\
\vdots \& \vdots \& \vdots \& \vdots\\
W^{k}_{n,1} \& W^{k}_{n,2} \& W^{k}_{n,3} \& W^{k}_{n,n} \\
};
\begin{pgfonlayer}{background}
\draw[rounded corners, dotted, fill=blue!15!white]
(m-1-1.north west) rectangle (m-1-3.south east);
\end{pgfonlayer}
\begin{pgfonlayer}{background}
\draw[rounded corners, dotted, fill=blue!15!white]
(j-1-2.north west) rectangle (j-5-2.south east);
\end{pgfonlayer}
\begin{pgfonlayer}{background}
\draw[rounded corners, dotted, fill=blue!15!white]
(j-1-4.north west) rectangle (j-5-4.south east);
\end{pgfonlayer}
\begin{pgfonlayer}{background}
\draw[rounded corners, dotted, fill=orange!15!white,fill opacity=0.5]
(j-1-1.north west) rectangle (j-1-4.south east);
\end{pgfonlayer}
\begin{pgfonlayer}{background}
\draw[rounded corners, dotted, fill=orange!15!white,fill opacity=0.5]
(j-3-1.north west) rectangle (j-3-4.south east);
\end{pgfonlayer}
\draw[->,thick,blue] (m-1-2.south) to [out=-10,in=270] (j-5-2.south);
\node[blue,below of=m-1-3,yshift=-1mm,xshift=1mm,rotate=-30] { $a^{k-1} W^k_{:,2}$};
\node[blue,above of=j-1-3,yshift=-4mm,xshift=0mm,rotate=0] {\tt  \hspace{-11mm} \cOne{}};
\node[orange,right of=j-3-4,yshift=0mm,xshift=0mm,rotate=-90] {\tt  \cTwoShort{}};
\end{tikzpicture}
\vspace{-8mm}
\caption{High-level idea of the sampling-based techniques.}
\label{fig:matrix}
\end{figure}

%% file: cone.tex
\section{Efficient Training by \cOne{}}\label{sec:cone}
\tightsubsection{Dropout}\label{sec:methods:dropout}
\citet{Dropout} introduced \dropout{}, a computationally efficient training approach that reduces the risk of overfitting. During each feedforward step, the algorithm picks a subset of nodes uniformly at random in each hidden layer and drops the remaining nodes temporarily. The sampled nodes are then used for feedforward evaluation and backpropagation.

While \dropout{} was originally introduced to fix overfitting, it introduced a computation reduction to the training process. In many cases, \dropout{} improved the runtime efficiency compared to the standard training process on the same architecture. However, there are scenarios in which training under \dropout{} requires more training iterations and eventually hurts the runtime. One can observe that due to the randomness in sampling with a fixed probability (usually $p = 1/2$), there is a risk of dropping nodes that significantly affect the output values. \citet{AdaptivDrop} addressed this issue by proposing \addropout{}, which uses a data-dependent distribution that is an approximation of the Bayesian posterior distribution over the model architecture and updates the sampling ratio adaptively w.r.t the current network. This method avoids randomly dropping significant nodes in the model.

\tightsubsection{Asymmetric Locality-Sensitive Hashing for Training Approximation}\label{sec:method:lsh}
Unlike in \dropout{}, one might want to intelligently select a small subset of so-called \emph{active nodes} for each layer for computing the inner products. In particular, given the vector $a^{k-1}$, the goal is to find a small portion of nodes $j$ in layer $k$ for which the value of $a^{k-1}W^k_{:,j}$ is maximized in order to avoid computing inner products for small values (estimating them as zero). Given a set $S$ of vectors (in this case, the set of columns in $W^k$) and a query vector $a$, the problem of finding a vector $w^* \in S$ with maximum inner product $\langle a, w^* \rangle$ is called maximum inner-product search (MIPS). To solve MIPS, \citet{MIPS} employ asymmetric locality-sensitive hashing (ALSH).
\begin{definition}[Asymmetric Locality-Sensitive Hashing \citep{MIPS}]\label{alsh}
Given a similarity threshold $S_0$ and similarity function $\Sim(\cdot)$, a family $\mathcal{H}$ of hash functions are $(S_0, cS_0, p_1, p_2)$-sensitive for $c$-NNS\footnote{$c$-approximation of nearest neighbor search \cite{LSHNN1}.} with $a\in \Reals^n $ as query and a set of $w \in \Reals^n$ vectors if for all $h \in \mathcal{H}$ chosen uniformly, the following conditions are satisfied:
\begin{gather*}
    \Sim(w, a) \geq \phantom{c}S_0 \quad \implies \Pr[h(Q(a)) = h(P(w))] \geq p_1 \\
    \Sim(w, a) \leq cS_0 \quad \implies \Pr[h(Q(a)) = h(P(w))] \leq p_2
\end{gather*}
\end{definition}
For $w$, $a \in \Reals^n$ with $\lVert w \rVert \leq C $, where $C$ is a constant less than 1, and $\lVert a \rVert = 1$, they define the transformations $P$ and $Q$ for the inner product as follows.
\begin{equation}\label{eq:PQ}
\begin{split}
    &P\colon \Reals^n \to \Reals^{n+m}, \quad w \mapsto \bigl[w; \lVert w \rVert ^{2^1}, \dots, \lVert w \rVert ^{2^m} \bigr] \\
    &Q\colon \Reals^n \to \Reals^{n+m}, \quad \mathrlap{a}\phantom{w} \mapsto \bigl[a; 1/2, \dots, 1/2 \bigr]
\end{split}
\end{equation}
In other words, to generate $P$, $w$ is padded with $m$ terms, where term $i$ is the $\ell_2$ norm of $w$ to the power of $2^i$. $Q$ is generated by padding $a$ with $m$ copies of the constant $1/2$. \citet{MIPS} prove that NNS in the transformed space is equivalent to the maximum inner product in the original space:
\begin{equation} \label{15}
    \argmax_w \langle w, a \rangle \approx \arg \min_w \lVert Q(a) - P(w)\rVert \,.
\end{equation}

Equation~\ref{15} motivates using MIPS for efficient training of DNNs. \citet{ScaleableLSH} build their algorithm (referred to as \textbf{\lsh{}} in this paper) upon Equation~\ref{15}. As explained in \S\ref{sec:prelim}, the feedforward step and backpropagation consist of many matrix multiplications, each of which involve a set of inner products as large as each hidden layer. \lsh{} uses ALSH to prune each layer by finding active nodes, in this case, nodes with maximum activation values. This is equivalent to solving MIPS in each layer.

Essentially, \lsh{} uses ALSH to find active nodes $j$ whose weight vector $W^k_{:,j}$ collides with an input vector $a^{k-1}$ under the same hash function. The probability of collision captures the similarity of vectors in each hidden layer. To do so, it sets the query vector as $q=a^{k-1}$ and the set of vectors using the columns of $W^k$ as $S = \set[\big]{W_{:,1}^k, \dots, W_{:,n}^k}$. Then, after constructing $Q$ and $P$ based on Equation~\ref{eq:PQ}, we have
\begin{equation}
\argmax_j \bigl\langle W_{:,j}^k, a^{k-1} \bigr\rangle \approx \argmin_j \bigl\lVert Q\bigl(a^{k-1}\bigr) - P\bigl(W_{:,j}^k\bigr) \bigr\rVert \,.
\end{equation}
\lsh{} constructs $L$ independent hash tables with $2^K$ hash buckets and assigns a $K$-bit randomized hash function to every table. Each layer has been assigned $L$ hash tables and a meta hash function to compute a hash signature for the weight vectors and fill all the hash tables before training. In this setting, $K$ and $L$ are tunable hyperparameters that affect the active set's size and quality.

During training, \lsh{} computes the hash signature of each incoming input using the existing hash functions. Then, a set of weight vectors will be returned using the hash values corresponding to the hash bucket. The active nodes in a layer are the union of their corresponding weight vectors from probing $L$ hash tables. Then, the model only performs the exact inner product on these active nodes and skips the rest. Finally, the gradient will only backpropagate through the active nodes and update the corresponding weights. In other words, ALSH is used to sample a subset of nodes with probability $1-(1-p^K)^L$ if $p$ is the probability of collision.
 
Updating the hash tables ensures that the modified weight vectors are recognized. Based on the results reported by \citet{ScaleableLSH}, the number of active nodes for each input can be as small as 5\% of the total nodes per layer. Thus, \lsh{} performs a significantly smaller set of inner products in each iteration. Moreover, due to the sparsity of the active sets belonging to different data inputs, the overlap between them throughout the dataset is small, so the weight gradient updates corresponding to these inputs are sparse as well. Thus, the hash table updates are executed after processing a batch of inputs and can be executed in parallel. The main advantage of \lsh{} is that, unlike \dropout{}, it finds the active nodes \emph{before} computing the inner products.

%% file: ctwo.tex
\tightsection{Efficient Training by \cTwo{}}\label{sec:ctwo}
While techniques discussed in \S\ref{sec:cone} reduce the vector-matrix multiplication time by selecting a subset of columns (nodes) from each weight matrix $W^k$ and computing the inner product exactly for them, an alternative approach is to select all columns but to compute inner products approximately. This idea has been proposed by \citet{FasterNN}. This paper is built on the MC method by \citet{MonteCarloMM} for fast approximation of matrix multiplication. We first review the work of \citet{MonteCarloMM} in \S\ref{sec:method:mm} and then in \S\ref{sec:method:fa} we explain how \citet{FasterNN} adapt the method to develop an algorithm for efficient training of DNNs.

\tightsubsection{Fast Approximation of Matrix Multiplication}\label{sec:method:mm}
For many applications, a fast estimation of the matrix product is good enough. In addition to hardware/software oriented optimizations such as cache management~\citep{GPUMM, anatomyMM} or half-precision computations~\citep{precisionMM, MM-CPU}, MC methods~\citep{montecarlo} have been used for such estimations. At a high level, MC methods use repeated sampling and the law of large numbers to estimate aggregate values.

Recall that given two matrices $A \in \Reals^{m \times n}$ and $B \in \Reals^{n \times p}$, the product $AB$ is an $m \times p$ matrix, where every element $AB_{i,j}$ is the inner product of $i$-th row of $A$ with the $j$-th column of $B$:
\begin{equation}\label{mm1}
    AB_{i,j} = \langle A_{i,:}^T, B_{:,j}\rangle= \sum^{n}_{t=1} A_{i,t} B_{t,j}
\end{equation}
In an MC estimation of $AB_{i,j}$, instead of computing the sum over all $t \in [1, n]$, only a small sample of elements $\sigma \subset \set{1, \dots, n}$, where $c = \lvert \sigma \rvert \ll n$, are considered.
Arguing that uniform sampling would add a high error in estimating $AB$, \citet{MonteCarloMM} introduce a nonuniform sampling method to generate $\sigma$ with a probability proportional to the magnitude of data. Specifically, they develop a randomized algorithm that samples each column $i$ of A and row $i$ of B with probability
\begin{equation}\label{eq:mm2}
p_i = \frac {\lVert A_{:,i} \rVert \cdot \lVert B_{i,:} \rVert}{\sum_{t=1}^n  \lVert A_{:,t} \rVert \cdot \lVert B_{t,:} \rVert} \,,
\end{equation}
where $\lVert \cdot \rVert$ is the $\ell_2$ norm.
They define the matrices $C \in \Reals^{m \times c}$ and $R \in  \Reals^{c \times p}$ as $C = ASD$ and $R = (SD)^T B$, respectively.
$S$ is then defined as an $n \times c$ sampling matrix, with $S_{ij} = 1$ if the $i$th row of $A$ is the $j$th sample.
$D$ is a $c \times c$ diagonal scaling matrix with $D_{jj} = \frac{1}{\sqrt{c p_j}}$. The authors prove that defining $p_i$ as in Equation~\ref{eq:mm2} minimizes the expected estimation error, $\mathbb{E}\bigl[\lVert AB - CR \rVert_F\bigr]$ , where $\lVert \cdot \rVert_F$ is the Frobenius norm.
Each element $AB_{i,j}$ is estimated as
$\sum^{c}_{t=1} C_{i,t} R_{t,j}  =   \sum^{c}_{t=1} \frac{1}{c p_i} A_{i,t} B_{t,j} \approx AB_{i,j}$. Sampling row-column pairs w.r.t $p_i$ reduces the complexity of matrix multiplication to $O(mcp)$.

\tightsubsection{\aprox}\label{sec:method:fa}
Training DNNs involves computationally expensive matrix multiplication operations. However, the gradients computed during backpropagation only approximate directions towards the optimal solution, so the training process has a high tolerance to small amounts of noise. This makes approximation of matrix multiplication a reasonable choice to speed up training of DNNs. Following this idea, \citet{FasterNN} propose a MC  approximation method for fast training of DNNs (in this paper, referred to as \textbf{\maprox{}} for the mini-batch setting and \textbf{\saprox{}} for the stochastic setting) based on the MC estimation of matrix multiplication explained in \S\ref{sec:method:mm}. Despite the fact that \citet{MonteCarloMM} provide an unbiased estimate for the matrix multiplication $AB$ (i.e., $E[CR] = AB$), \citet{FasterNN} prove that the sampling distribution is not able to provide an unbiased estimation of the weight gradient if it is used for both the forward step and backward pass simultaneously.

One way to eliminate the bias is to use MC approximation only in the forward pass, propagate the gradient through the entire network, and perform the exact computations. However, experiments show this approach results in lower accuracy in practice. Therefore, \citet{FasterNN} propose a new sampling distribution that yields an unbiased estimate of the weight gradient $\nabla_W \hat{\mathcal{L}}$ when it is used only during the feedforward step. Specifically, they sample column-row pairs independently from $A \in \Reals^{m \times n}$ and $B \in \Reals^{n \times p}$.

Let $k$ be the number of samples for estimation, let $V \in \Reals^{n \times n}$ be a diagonal sampling matrix with $V_{i,i} = Z_i \sim \operatorname{Bernoulli}(p_i)$, where $\sum^n_{i=0} p_i = k$, and let $D \in \Reals^{n \times n}$ be a diagonal scaling matrix with $D_{i, i} = \frac{1}{\sqrt{p_i}}$. Then, the multiplication of matrices $A$ and $B$ can be estimated as
$AB \approx \sum_{i=0}^n \frac{Z_i}{p_i} A_{:,i} B_{i,:} = AVDD^{T}V^{T}B = A'B'$.
The estimation error in this case is $E\bigl[\lVert AB-A'B'\rVert^2_F\bigr] = \sum_{i=0}^n \frac{1-p_i}{p_i} \lVert A_{i}\rVert^2 \lVert B_{i}\rVert^2$. 
Hence, the following assignment of probabilities minimizes the estimation error under the constraint $\sum_{i=0}^n p_i = k$.
\begin{equation}\label{p_i}
    p_i = \min \set[\bigg]{ \frac{k \lVert A_i\rVert \lVert B_i\rVert}{\sum^n_{t=0} \lVert A_t\rVert \lVert B_t\rVert },\, 1 }  
\end{equation}
The authors prove that training a neural network by approximating matrix products in backpropagation converges with the same rate as standard SGD and results in an unbiased estimator when nonlinearities are not considered. When accounting for nonlinearities, the results hold as long as the MC approximation of $Wx$ is unbiased and the activation and loss functions are $\beta$-Lipschitz.

%% file: analysis.tex
\tightsection{Theoretical Analysis}\label{sec:analysis}
As explained in \S\ref{sec:method:lsh} and \S\ref{sec:method:mm}, sampling-based approaches seek to speed up the training of DNNs by skipping a large number of computations and approximating matrix multiplications. In this section we provide negative theoretical results for scalability against the feedforward approximation. At a high level, we show that small estimation errors in the initial layers get propagated and compounded in subsequent layers. \citet{FasterNN} already observed the low performance of \aprox{} when the feedforward step is approximated and therefore only applied approximation during backpropagation for MLPs. As such, we focus on \lsh{} for our analysis. First, let us introduce the following notation, which we will use throughout this section.
\begin{itemize}
    \item $\bar{a}^k_j$: the estimation of $a^k_j$ by \lsh{}.
    \item $e^k_j = a^k_j - \bar{a}^k_j$: the activation value estimation error.
    \item $\actv^k_j$: the set of active nodes for $n_j^k$.
\end{itemize}

\begin{lemma}\label{lem:er}
Let $f$ be a linear activation function such that $a = f(z)=z$. Assuming the active nodes are detected exactly, the estimation error for the node $n_j^k$ by \lsh{} is as follows.
\[
e^k_j = \begin{cases}
\sum\limits_{i\notin \actv^1_j} x_i W^1_{i,j} \quad &\text{if\/}\ k=1 \\
e^{k-1}W^k_{:,j}+\sum\limits_{i\notin \actv^k_j} \bar{a}^{k-1}_i W^1_{i,j} \quad &\text{otherwise}
\end{cases}
\]
\end{lemma}
\begin{proof}
We want to show that the estimation error for the node $n_j^k$ by \lsh{} is as follows.
\[
e^k_j = \begin{cases}
\sum_{i\notin \actv^1_j} x_i W^1_{i,j} \quad &\text{if\/}\ k=1 \\
e^{k-1}W^k_{:,j}+\sum_{i\notin \actv^k_j} \bar{a}^{k-1}_i W^1_{i,j} \quad &\text{otherwise}
\end{cases}
\]

First, for $k=1$:
\begin{align*}
    &a^1_j = \sum_{i=1}^nx_iW^1_{i,j} = \sum_{l\in\actv^k_j}x_lW^1_{l,j} + \sum_{i\notin\actv^k_j}x_iW^1_{i,j}= \bar{a}^1_j + \sum_{i\notin\actv^k_j}x_iW^1_{i,j}\\
    \implies &e^1_j=a^1_j- \bar{a}^1_j=\sum_{i\notin\actv^k_j}x_iW^1_{i,j}
\end{align*}
Analogously, when $k>1$: 
\begin{align*}
    a^k_j &= \sum_{i=1}^n a^{k-1}_iW^{k-1}_{i,j} = \sum_{i=1}^n (\bar{a}^{k-1}_i+e^{k-1}_i)W^{k-1}_{i,j}\\
    &= \sum_{l\in\actv^k_j} \bar{a}^{k-1}_lW^{k-1}_{l,j} + \sum_{i\notin\actv^k_j} \bar{a}^{k-1}_iW^{k-1}_{i,j} + \sum_{i=1}^n e^{k-1}_iW^{k-1}_{i,j}\\
    &= \bar{a}^k_j + e^{k-1}W^k_{:,j}+\sum_{i\notin \actv^k_j} \bar{a}^{k-1}_iW^{k-1}_{i,j}\\
    \implies e^1_j&=a^1_j- \bar{a}^1_j= e^{k-1}W^k_{:,j}+\sum_{i\notin \actv^k_j} \bar{a}^{k-1}_iW^{k-1}_{i,j}
\end{align*}
\end{proof}

Lemma~\ref{lem:er} provides a recursive formula for the activation value estimation error in terms of the weighted summation over active nodes versus inactive nodes. To provide a non-recursive and easier to understand formula, in Theorem~\ref{th:er} we assume that the weighted summation over the active nodes is always $c$ times that of the inactive nodes.
\begin{theorem}\label{th:er}
Let $f$ be a linear activation function such that $a = f(z)=z$. Suppose for any node $n_p^l$,
\[\sum_{i\in\actv^l_p}a^{l-1}_iW_{i,p}=c\sum_{i\notin\actv^l_p}a^{l-1}_iW_{i,p} \,.\]
Then, $a^k_j = \bar{a}^k_j\Bigl(\frac{c+1}{c}\Bigr)^k$. That is, $e^k_j=\bar{a}^k_j\Bigl(\bigl(\frac{c+1}{c}\bigr)^k-1\Bigr)$.
\end{theorem}

\begin{proof}
For any node $n_p^l$, we have
\[\sum_{i\in\actv^l_p}a^{l-1}_iW_{i,p}=c\sum_{i\notin\actv^l_p}a^{l-1}_iW_{i,p} \,.\]
We then use induction to prove $a^k_j = \bar{a}^k_j\Bigl(\frac{c+1}{c}\Bigr)^k$.

\emph{Base case.} When $k=1$:
\begin{equation*}
    a^1_j
    = \sum_{i=1}^nx_iW^1_{i,j}
    = \sum_{l\in\actv^k_j}x_lW^1_{l,j} + \sum_{i\notin\actv^k_j}x_iW^1_{i,j}
    = \bar{a}^1_j + \frac{1}{c}\bar{a}^1_j
    = \bar{a}^1_j\frac{c+1}{c}
\end{equation*}

\emph{Inductive step.} Assuming $a^{k-1}_j = \bar{a}^{k-1}_j\Bigl(\frac{c+1}{c}\Bigr)^{k-1}$:
\begin{align}\label{eq:tmp1}
 a^k_j &= \sum_{i=1}^n a^{k-1}_iW^{k-1}_{i,j} = \sum_{i=1}^n (\bar{a}^{k-1}_i+e^{k-1}_i)W^{k-1}_{i,j}\\
\nonumber &= \sum_{l\in\actv^k_j} \bar{a}^{k-1}_lW^{k-1}_{l,j} + \sum_{i\notin\actv^k_j} \bar{a}^{k-1}_iW^{k-1}_{i,j} + \sum_{i=1}^n e^{k-1}_iW^{k-1}_{i,j}\\
\nonumber &= \bar{a}^k_j + \frac{1}{c} \bar{a}^k_j + \sum_{i=1}^n e^{k-1}_iW^{k-1}_{i,j}
     = \frac{c+1}{c}\bar{a}^k_j + \sum_{i=1}^n e^{k-1}_iW^{k-1}_{i,j}
\end{align}

Let $A=\sum_{i=1}^n e^{k-1}_iW^{k-1}_{i,j}$.
We have
\begin{align*}
    e^{k-1}_i &= a^{k-1}_i - \bar{a}^{k-1}_i = a^{k-1}_i - a^{k-1}_i \Bigl(\frac{c}{c+1}\Bigr)^{k-1} \\
    &= a^{k-1}_i\biggl(1 - \Bigl(\frac{c}{c+1}\Bigr)^{k-1}\biggr) \,.
\end{align*}
Thus,
\begin{align*}
    A&=\sum_{i=1}^n e^{k-1}_iW^{k-1}_{i,j} = \sum_{i=1}^n a^{k-1}_i\biggl(1 - \Bigl(\frac{c}{c+1}\Bigr)^{k-1}\biggr)W^{k-1}_{i,j} \\
    &=  \biggl(1 - \Bigl(\frac{c}{c+1}\Bigr)^{k-1}\biggr)\sum_{i=1}^n a^{k-1}_iW^{k-1}_{i,j}
    =  \biggl(1 - \Bigl(\frac{c}{c+1}\Bigr)^{k-1}\biggr)a^k_j \,.
\end{align*}
Now, plugging $A$ back into Equation~\ref{eq:tmp1}, we get:
\begin{align*}
    &a^k_j = \frac{c+1}{c}\bar{a}^k_j + A
    = \frac{c+1}{c}\bar{a}^k_j + \biggl(1 - \Bigl(\frac{c}{c+1}\Bigr)^{k-1}\biggr)a^k_j \\
    \implies &a^k_j \biggl(1-\biggl(1 - \Bigl(\frac{c}{c+1}\Bigr)^{k-1}\biggr)\biggr) = \frac{c+1}{c}\bar{a}^k_j \implies a^k_j = \bar{a}^k_j\Bigl( \frac{c+1}{c} \Bigr)^k
\end{align*}
\end{proof}

Theorem~\ref{th:er} proves that the estimation error increases \textbf{exponentially} with the number of layers. As a result, due to the sharp increase in the estimation error, \lsh{} does not scale to DNNs. To better observe this, suppose $c=5$ (i.e., the weighted sum for the active nodes is five times that of the inactive nodes). Then, using Theorem~\ref{th:er}, the error-to-estimate ratios for different numbers of layers are as follows.
\begin{center}
\begin{tabular}{|c|c|c|c|c|c|c|}
\hline
$\mathbf{k}$ & 1 & 2 & 3 & 4 & 5 & 6 \\
\hline
$\mathbf{{e^k_j}/{\bar{a}^k_j}}$ & 0.2 & 0.44 & 0.72 & 1.07 & 1.48 & 1.98 \\
\hline
\end{tabular} 
\end{center}
From the above table, it is evident that as soon as the depth of the network gets larger than 3, the estimation error dominates the estimation value. This is consistent with our experiment results, where \lsh{} failed to scale with more than 3 hidden layers.

%% file: exp_setup.tex
\tightsection{Experiment Setup}\label{sec:exp_setup}
\tightsubsection{Hardware}
This paper aims to evaluate sampling-based approaches for training DNNs on regularly available machines; thus, we ran all experiments on a single-CPU machine (Intel Core i9-9920X machine with 128 GB of memory) without a GPU.

\tightsubsection{Datasets}\label{sec:exp_setup:datasets}
We used the following six benchmark datasets for our experiments.
\begin{description}
\item[MNIST \citep{mnist}] 70,000 handwritten digits, each in the form of a $28 \times 28$ grayscale image, and 10 classes (digits zero to nine).
\item[Kuzushiji-MNIST \citep{kmnist}] 70,000 cursive Japanese characters, each in the form of a $28 \times 28$ grayscale image, and 10 classes.
\item[Fashion-MNIST \citep{fashionmnist}] 70,000 fashion products, each in the form of a $28 \times 28$ grayscale image, and 10 classes.
\item[EMNIST-Letters \citep{emnist}] 145,600 handwritten letters, each in the form of a $28 \times 28$ grayscale image, and 26 classes.
\item[NORB \citep{Norb}] 48,600 photographs of 50 toys from different angles, each in the form of a $96 \times 96$ grayscale image, and 5 classes.  
\item[CIFAR-10 \citep{CIFAR10}]  60,000 color images, each of dimensions $32 \times 32$, and 10 classes.  
\end{description}
We randomly partition the datasets as shown below:

\begin{center}
\small
\begin{tabular}{cccc}
    \toprule
    Dataset & Train & Test & Validation \\
    \midrule
    (Kuzushiji/Fashion-)MNIST & 55000 & 10000 & 5000 \\
    EMNIST-Letters & 104800 & 20800 & 20000 \\
    NORB & 22300 & 24300 & 2000 \\
    CIFAR-10 & 45000 & 10000 & 5000 \\
    \bottomrule
\end{tabular}
\end{center}

We obtained similar results across different datasets. For brevity, we provide a detailed discussion only using the results on the MNIST dataset, but extensive results for other datasets are provided in our technical report~\citep{techrep}.

\tightsubsection{Methods Evaluated}
We evaluated four sampling-based approaches for training DNNs discussed in \S\ref{sec:cone} and \S\ref{sec:ctwo}, namely \dropout{} \citep{Dropout}, \addropout{} \citep{AdaptivDrop}, \aprox{} \citep{FasterNN}, and \lsh{} \citep{ScaleableLSH}, on fully connected DNNs. In addition, the regular training approach, referred to as \reg{}, has been implemented for comparison purposes. All implementations are in Python 3.9 using the PyTorch library. For \aprox{} \footnote{\url{github.com/acsl-technion/approx}}, \lsh{} \footnote{\url{github.com/rdspring1/LSH-Mutual-Information}}\footnote{\url{github.com/rdspring1/LSH_DeepLearning}}, \dropout{} \footnote{\label{footnote:dropout}\url{github.com/gngdb/adaptive-standout}}, and \addropout{}\footnote{see footnote~\ref{footnote:dropout}.} we used the publicly available code.

\tightsubsection{Default Values}\label{sec:exp_setup:defaults}
To train our models, we use SGD. The activation function used for hidden layers is ReLU due to its simplicity, ease of computation, and the fact that it helps with the vanishing gradients problem \citep{DL}. The output layer activation function is log softmax, and the loss function used throughout experiments is the negative log-likelihood. The learning rate is always either $10^{-4}$ or $10^{-3}$ depending on the setting, and the models are trained for 50 epochs. In particular, we set the learning rate to $10^{-4}$ for \saprox{}\footnote{Across all tables and plots we use the subscripts ``S'' and ``M'' to refer to the SGD and mini-batch SGD (with default batch size 20) settings, respectively. When there is no subscript, the default is MGD for \aprox{} and SGD for all other methods.}. The hyperparameters of our implementation are the best values reported for each approach. For \aprox{} the batch size is set to 20 and $k = 10$. For \lsh{}, $K = 6$, $L = 5$, and $m = 3$ (Equation~\ref{eq:PQ}) as specified in \citep{ScaleableLSH}. In order to have a fair comparison with \lsh{}, we set the probability of picking nodes for \dropout{} and \addropout{} to $p = 0.05$, and we use a network with 3 hidden layers and 1000 hidden units per layer across algorithms. The implementation of \lsh{} provided in \citep{MIhashing} performs better with the Adam optimizer \citep{Adam} than with Adagrad \citep{JMLR:v12:duchi11a}, which is used in the original implementation in \citep{ScaleableLSH}. Hence, we use Adam in our experiments.

For the convolutional setting, we used ResNet-18 with two fully-connected layers as a classifier to run our experiments. We limit the approximation to the classifier and keep the convoluted operations exact. Also, for CIFAR-10, we use pure SGD instead of Adam.

\tightsubsection{Experiment Plan}\label{sec:exp_setup:plan}
We are mainly interested in evaluating the following.
\begin{description}
\item[Accuracy]
How do the algorithms perform when training networks with different depths?
\item[Time]
How scalable are the evaluated algorithms (in particular, \lsh{} and \aprox{}) w.r.t training time?
\item[Hyperparameters]
How do hyper-parameter choices (e.g., batch size) affect training time and accuracy?
\end{description}
Accuracy here refers to the percentage of correct predictions on the entire dataset. Since we focus on multi-class classification, we also provide confusion matrices.

%% file: results.tex
\section{Experiment Results}\label{sec:results}
\subsection{Accuracy}
We begin our experiments by addressing the first question in \S\ref{sec:exp_setup:plan}. To do so, we generate models with different numbers of hidden layers (1 to 7) and evaluate each method on all six datasets discussed in \S\ref{sec:exp_setup:datasets} for both stochastic\footnote{ when the batch size is equal to 1.} and mini-batch settings. The confusion matrices for all algorithms are provided in Figure~\ref{fig:matrices}. Every row in the figure shows the performance of an algorithm, while different columns represent networks with different numbers of hidden layers. In all plots contained within the figure, the x-axis shows the model prediction and the y-axis shows the true labels. Consequently, the diagonal cells show correct predictions, while all other cells are incorrect predictions. Ideally, the models should have (near-)zero values on non-diagonal cells.

\input{figs/fig-tab}

\begin{figure*}[bhpt]
\begin{minipage}[t]{0.32\linewidth}
\includegraphics[width=\linewidth]{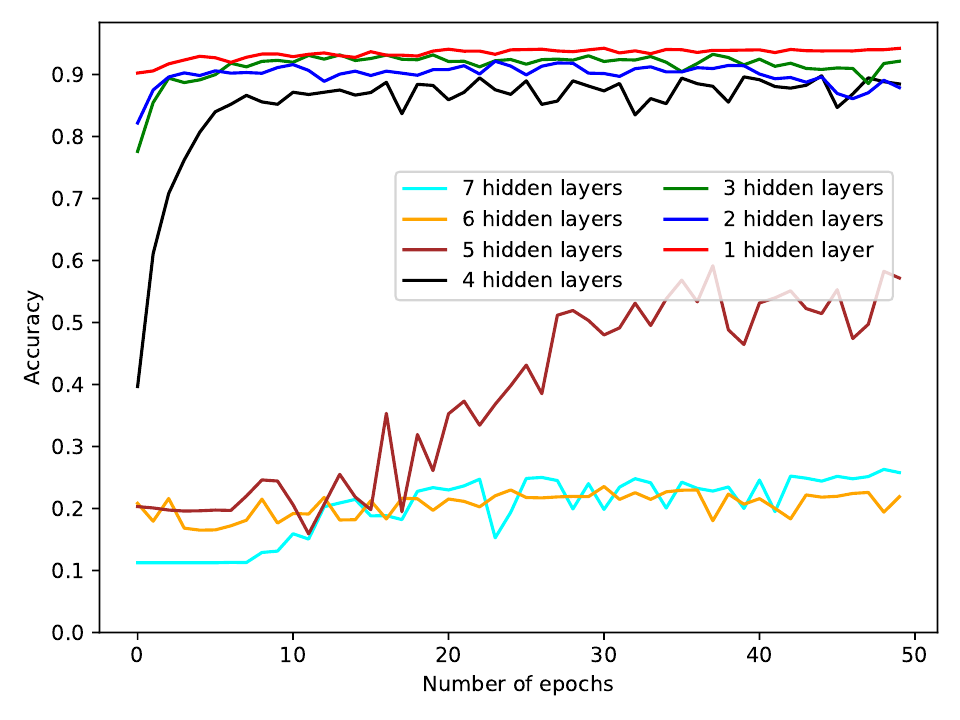}
\caption{\lsh{}: validation accuracy for different numbers of layers.}
\label{fig:acc-lsh-l}
\end{minipage}
\hfill
\begin{minipage}[t]{0.32\linewidth}
\centering
\includegraphics[width=\linewidth]{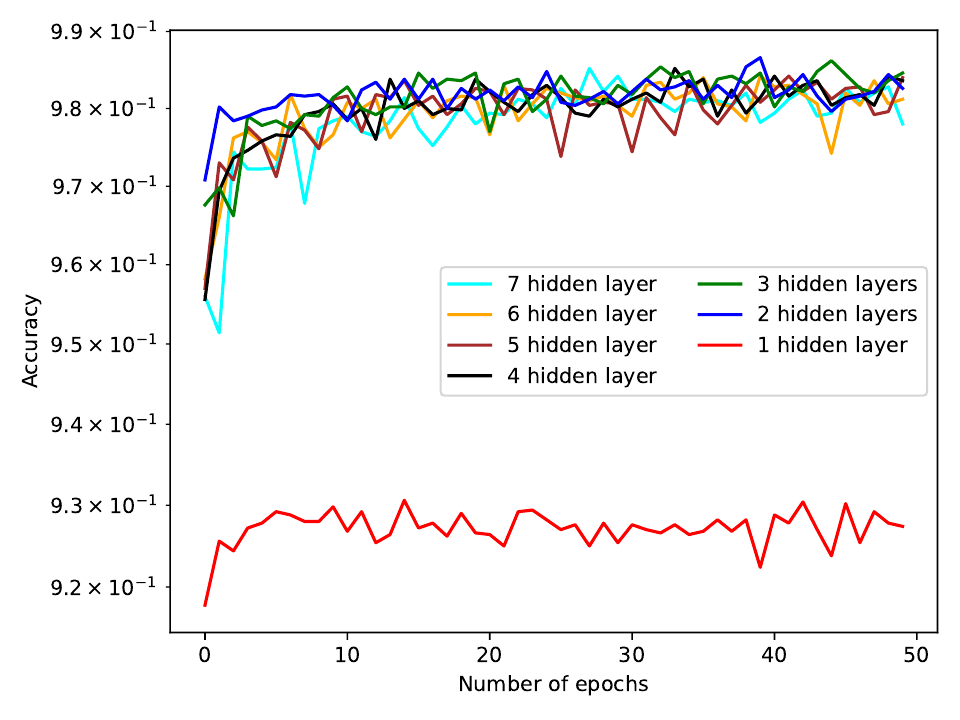}
\caption{\aprox{}: validation accuracy for different numbers of layers.}
\label{fig:acc-approx-l}
\end{minipage}
\hfill
\begin{minipage}[t]{0.32\linewidth}
\centering
\includegraphics[width=\linewidth]{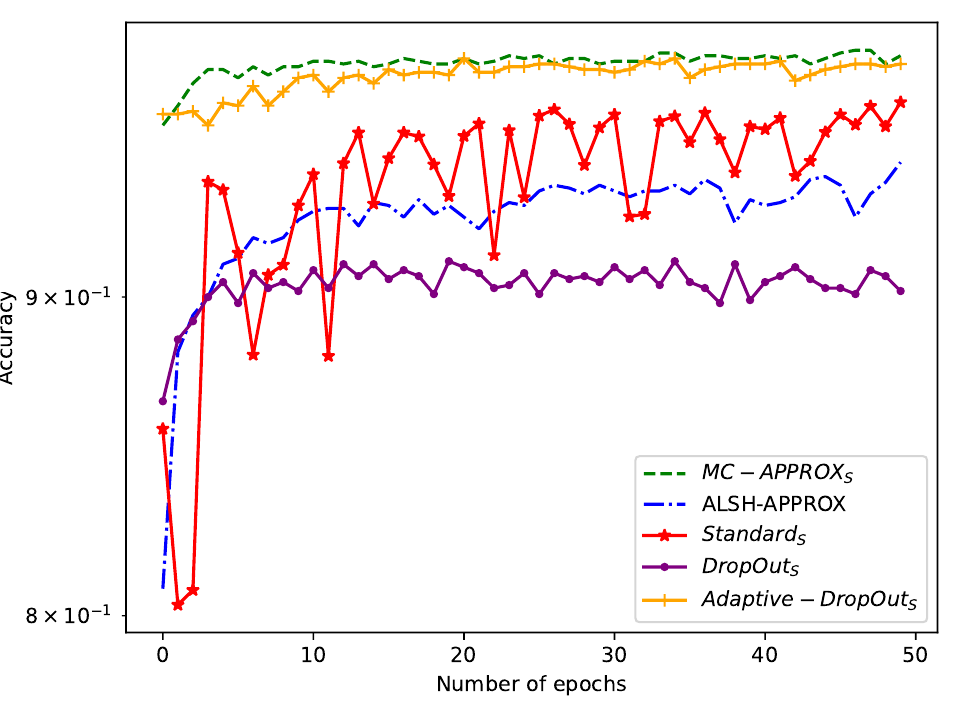}
\caption{Validation accuracies with 3 hidden layers. }
\label{fig:acc-3all}
\end{minipage}
\end{figure*}

\begin{figure*}[bhpt]
\begin{minipage}[t]{0.32\linewidth}
\centering
\includegraphics[width=1.1\linewidth]{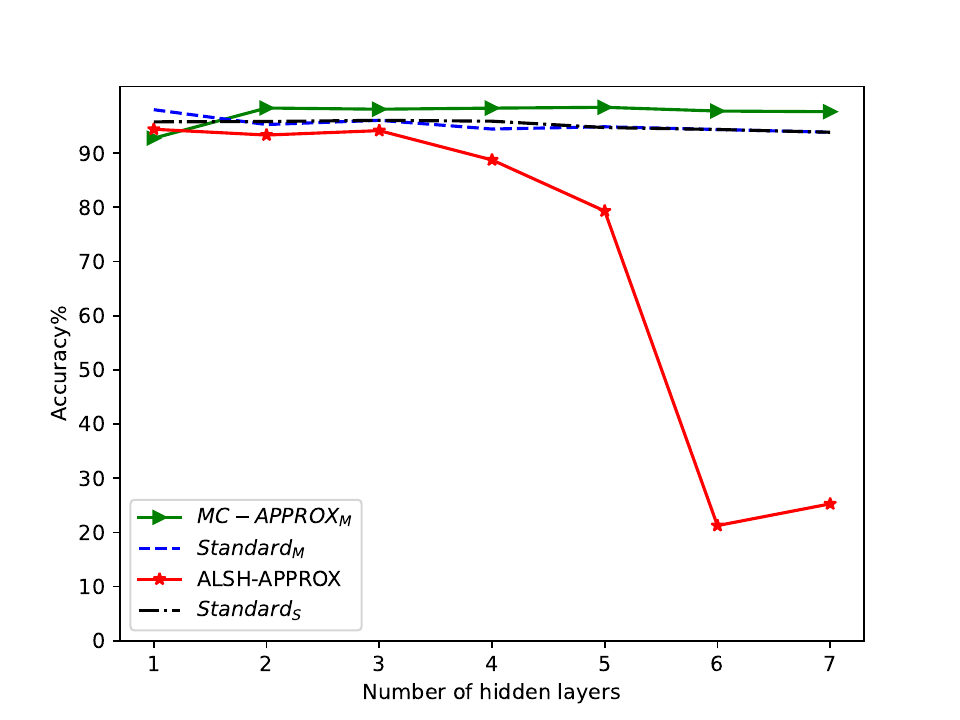}
\caption{Test accuracy for different numbers of hidden layers.}
\label{fig:acc-lsh-approx}
\end{minipage}
\hfill
\begin{minipage}[t]{0.32\linewidth}
\centering
\includegraphics[width=1.1\linewidth]{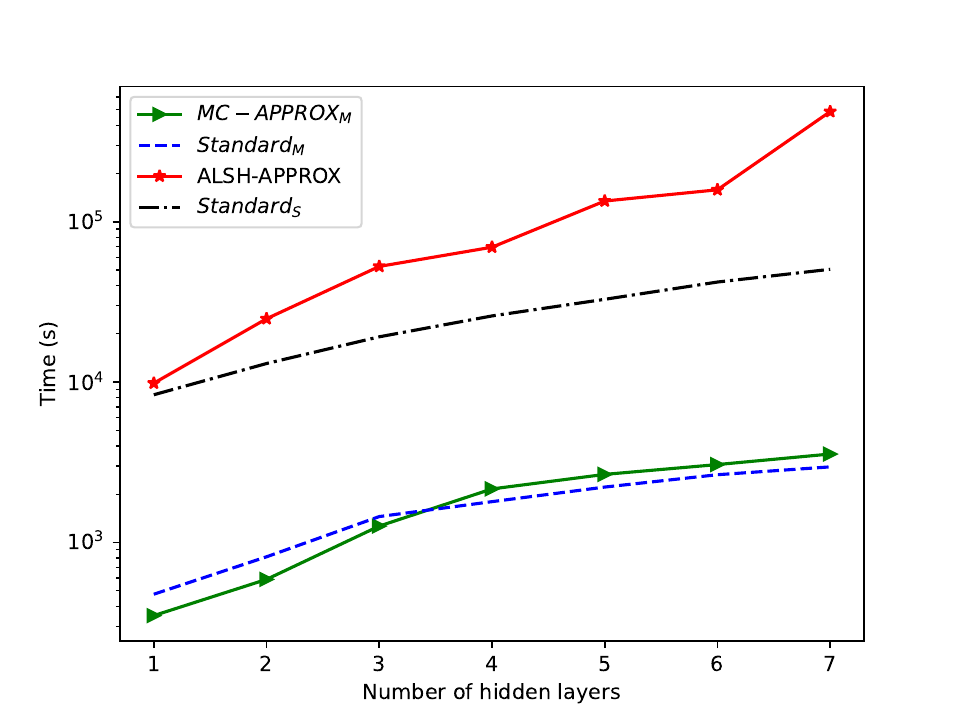}
\caption{Training time for different numbers of hidden layers.}
\label{fig:time-all-layers}
\end{minipage}
\hfill
\begin{minipage}[t]{0.32\linewidth}
\centering
\includegraphics[width=1.1\linewidth]{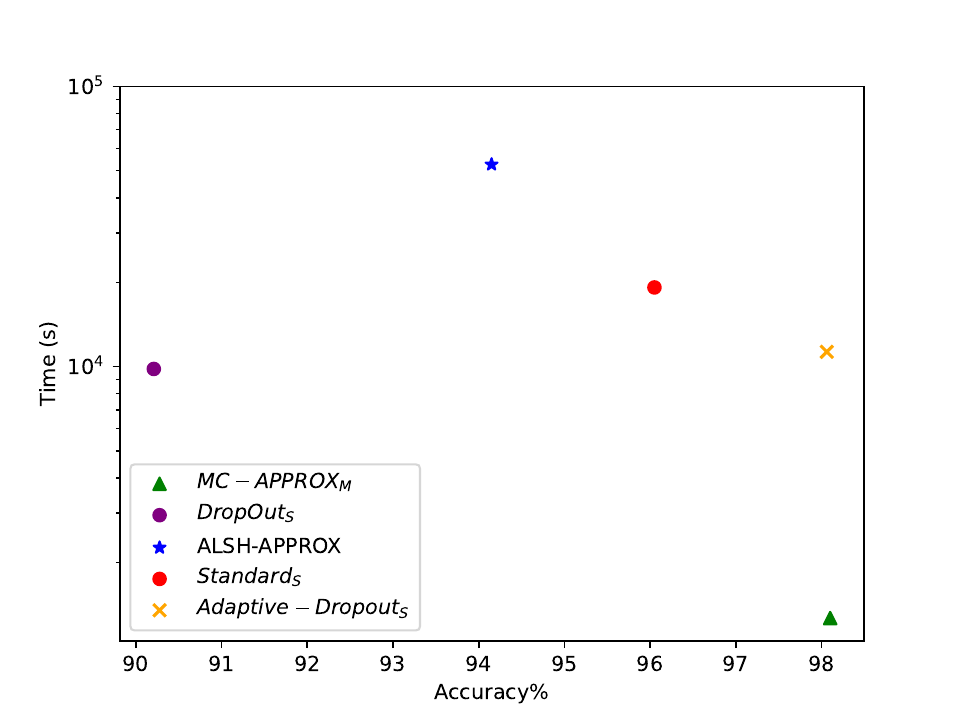}
\caption{Accuracy vs. training time (3 hidden layers).}
\label{fig:acc_vs_time}
\end{minipage}
\end{figure*}

\stitle{Baselines}
\reg{}\footnote{training the neural network without sampling.}, \dropout{}, and \addropout{} (the first three rows in Figure~\ref{fig:matrices}) are our baselines for comparisons. For \dropout{}, the nodes are sampled randomly with probability $p$, and for \addropout{}, $p$ is updated w.r.t the Bayesian posterior distribution of data input. In standard feedforward training, we expect to observe a decrease in generalization error over complex datasets as we add layers to the network and an increase in the ability to learn nonlinear functions. Clearly, this does not include the cases in which \reg{} overfits.

\stitle{\lsh{}}
The experiment results on \lsh{} (Row 4 in Figure~\ref{fig:matrices}) confirm a decrease in accuracy as the number of layers increases. In particular, Figures~\ref{fig:lsh-l5}~and~\ref{fig:lsh-l7} show a sharp decrease in performance on 5 to 7 layers. This is confirmed in Figure~\ref{fig:acc-lsh-approx}, where the accuracy of \lsh{} drops from 70.07\% to 25.14\% from 5 to 7 layers. Comparing \lsh{} with \sreg{} in Figure~\ref{fig:matrices}, though initially the two algorithms perform similarly on a small number of layers (Columns 1 and 2), the performance gap quickly increases with the number of layers --- confirming the \emph{lack of scalability} of \lsh{} for DNNs. This is also observed in Figure~\ref{fig:acc-lsh-approx}. We can see from Table~\ref{tab:acc} that, compared with \sdropout{}, \lsh{} performs better, but it does not outperform \sreg{} and \saddropout{}.

\stitle{\aprox{}}
\aprox{} is designed for use with mini-batch SGD. As we shall further investigate in \S\ref{sec:exp:hyper}, even though \saprox{} outperforms other methods evaluated (Table~\ref{tab:acc}), the runtime for large numbers of layers and datasets is so high that it is infeasible for computation-limited systems. This is reflected in Table~\ref{tab:time-mb1}. Therefore, as indicated in \S\ref{sec:exp_setup:defaults}, we use mini-batch SGD (with batch size 20) as the default setting in our experiments. The experiment results on \maprox{} are provided in the last row of Figure~\ref{fig:matrices}. \maprox{} shows equally good performance across different numbers of layers, confirming its scalability for DNNs. In particular, when varying the number of hidden layers (Figure~\ref{fig:acc-lsh-approx}), the minimum accuracy obtained by \maprox{} is 92.71\% for one hidden layer. Comparing the confusion matrices of \aprox{} with \addropout{} and \reg{} in Figure~\ref{fig:matrices}, we can see that performance is consistent across the three algorithms. As shown in Figure~\ref{fig:acc-lsh-approx}, in most cases, \maprox{} outperformed \mreg{} with 2--4\% difference in accuracy. This is also evident in Table~\ref{tab:acc}, where \maprox{} and \saprox{} outperformed other algorithms on the MNIST and Fashion-MNIST datasets with 3 hidden layers. To further assess the scalability of \maprox{} in deeper networks, we conducted additional experiments using 10- and 20-layer networks. The obtained accuracy rates of 97.32\% for 10 layers and 95.71\% for 20 layers validate our earlier findings. Finally, looking at Figure~\ref{fig:acc-lsh-approx}, the only case in which \maprox{} fails to obtain the highest accuracy compared with \lsh{} is when the model has only 1 hidden layer. \lsh{} performs (94.4\%) slightly better than \maprox{} (92.71\%).

\begin{table*}[bhpt]
\centering
\small
\caption{Test accuracy (\%) for a network with 3 hidden layers.}
\begin{tabular}{ ccccccc }
\toprule
Dataset & \lsh{} & \maprox{} & \saprox{} & \sdropout{} & \saddropout{} & \sreg{} \\
\midrule
MNIST & 94.15 & 98.10 & \textbf{98.38}  & 90.21 & 98.06 & 96.46 \\
Kuzushiji-MNIST & 72.87 & 91.78 & \textbf{96.50} & 9.84 & 90.73 & 83.86 \\
Fashion-MNIST & 78.11 & 87.85 & \textbf{88.58} & 76.28 & 86.12 & 73.64 \\
EMNIST-Letters & 64.97 & 89.84 & \textbf{90.75} & 4.96 & 89.50 & 85.34 \\
NORB & 78.57 & 92.05 & \textbf{97.52} & 32.73 & 96.60 & 51.61 \\
CIFAR-10 & 10.31 & 73.26 & 62.11 & 67.85 & 75.55 & \textbf{93.02} \\
\bottomrule
\end{tabular}
\label{tab:acc}
\end{table*}

 \begin{table*}[pt]
\small
\caption{Training time per epoch (sec.) with 3 hidden layers and batch size 1 on MNIST.}
\centering
\begin{tabular}{cccccc} 
\toprule
Method & \lsh{} & \saprox{} & \sdropout{} & \saddropout{} & \sreg{} \\
\midrule
Total & $807.50 \pm 22.92$ & $422.23 \pm 3.36$ & $196.15 \pm 0.55$ & $225.85 \pm 1.91$ & $361.51 \pm 5.13$ \\
Feedforward & $168.02 \pm 3.34$ & $28.44 \pm 0.076 $ & $32.32 \pm 0.04$ & $59.03 \pm 0.10$ & $28.976 \pm 0.04$ \\
Backpropagation & $356.16 \pm 7.28$ & $110.98 \pm 0.69 $ &$131.02 \pm 0.57$ & $132.68 \pm 1.89$ & $61.88 \pm 1.02$ \\
\bottomrule
\end{tabular}
\label{tab:time-mb1}
\end{table*}

\begin{table*}[pt]
\small
\caption{Training time per epoch (sec.) with 3 hidden layers and batch size 20 on MNIST.}
\centering
\begin{tabular}{ccccc} 
\toprule
Method & \maprox{} & \mdropout{}  & \maddropout{} & \mreg{} \\
\midrule
Total & \textbf{$35.141 \pm 0.385$} & $42.818 \pm 0.108$  & $53.412 \pm 0.301$ & $39.310 \pm 1.087$ \\
Feedforward & $4.733\pm 0.054$ & $4.789 \pm 0.010$ & $10.539 \pm 0.024$ & $4.213 \pm 0.052$ \\
Backpropagation & $11.323 \pm 0.071$ & $7.957 \pm 0.029$ & $8.493 \pm 0.271$ & $23.143 \pm 1.080$ \\
\bottomrule
\end{tabular}
\label{tab:time-mb20}
\end{table*}

\subsection{Training Time}
After studying the impact of network depth on accuracy, we next turn our attention to efficiency (i.e., training time). The results from all five methods on three hidden layers, on one CPU and without parallelization, are summarized in Table~\ref{tab:time-mb1} and Table~\ref{tab:time-mb20}.
Even though from Table~\ref{tab:time-mb20} it is evident that \maprox{} significantly outperforms other approaches with batch size 20, \saprox{} is slower than \saddropout{}, \sreg{}, and \sdropout{}. Similarly, Figure~\ref{fig:time-all-layers} shows the runtime comparison of \maprox{} and \lsh{} with \sreg{} and \mreg{} (baseline) for different numbers of layers. The results confirm the superiority of \maprox{} over the other algorithms up to 3 layers. Note that the observed increase in the training time of \addropout{} per epoch in comparison to \reg{} can be attributed to the additional computational overhead of the construction of dropout masks and their subsequent multiplication with the weight matrices in each layer (Table~\ref{tab:time-mb20}).

\stitle{\lsh{}} 
Before discussing our efficiency results from \lsh{}, let us emphasize that \lsh{} is a scalable algorithm that significantly 
\emph{benefits from parallelization}  \citep{ScaleableLSH}. During training, the hash table construction, computing hash signature, querying hash tables, and updating weight vectors by sparse weight gradients are parallelized, which makes the algorithm fast with parallel processing using multiple processing units. We refer interested readers to the details and results of \citet{ScaleableLSH}.

\lsh{} needs to reconstruct the hash tables after a set of weight updates. Following the original implementation of \lsh{}, in our experiments, for the first 10000 training data points, we reconstruct hash tables every 100 images. Then gradually, we expand the set to avoid time-consuming table reconstructions and update the tables every 1000 images. This helps with directing the gradient and decreasing the redundancy in the dataset. Table~\ref{tab:time-mb1} shows that \lsh{} exhibits slower performance compared to all other methods in the statistic setting. 

Also, in models with additional hidden layers, we can see an increase in training time as shown in Figure~\ref{fig:time-all-layers} that is larger than other methods on the same network structure. This is consistent with the results presented by \citet{ScaleableLSH}, where it is shown that the runtime significantly drops with parallelization. Evidently, as shown by \citet{ScaleableLSH}, parallelization has no impact on the accuracy of \lsh{}. In other words, the accuracy scalability results of \lsh{} discussed in the previous section are independent of parallelization.

\stitle{\aprox}
Due to the sampling ratio of \aprox{} ($p \approx 0.1$), \aprox{} performs more atomic scalar operations than \lsh{} with 5\% of the nodes. However, based on the results from Figure~\ref{fig:time-all-layers} and Table~\ref{tab:time-mb1}, \maprox{} and \saprox{} are around 20 and 2 times faster than \lsh{}, respectively. This is because of the significantly lower overhead of \aprox{} compared to \lsh{}. Figure~\ref{fig:time-all-layers} shows that \maprox{} outperforms other methods while maintaining a training time comparable to \mreg. Notably, despite the additional computational load in the backpropagation process with \maprox{}, it achieves backpropagation times that are twice as fast as those of \reg{} on the same architecture (Table~\ref{tab:time-mb20}). For networks with fewer than 4 hidden layers, \maprox{} is slightly faster than \mreg{}, and for deeper networks, the opposite is true. Nevertheless, Figure~\ref{fig:acc-approx-l} confirms the higher accuracy of the \maprox{} for various numbers of layers on MNIST. From Table~\ref{tab:time-mb1}, it is evident that \sreg{} is faster than \saprox{}. The reason is that, in order to estimate probabilities based on Equation~\ref{p_i} for each mini-batch, \aprox{} makes a pass over the mini-batch and the matrix $W$. As a result, in SGD, where mini-batch size is one, the overhead time and the time to approximate the matrix multiplication exceeds the required time for exact multiplication (\sreg). In \S\ref{mem-analysis}, we provide a thorough analysis of the runtime comparison between \dropout{}, \addropout{}, and \maprox{}.

Finally, Figure~\ref{fig:acc_vs_time} shows that \aprox{} performs better in both speed and accuracy compared to other methods.

\begin{figure*}
\begin{minipage}[t]{0.32\linewidth}
\centering
\includegraphics[width=\linewidth]{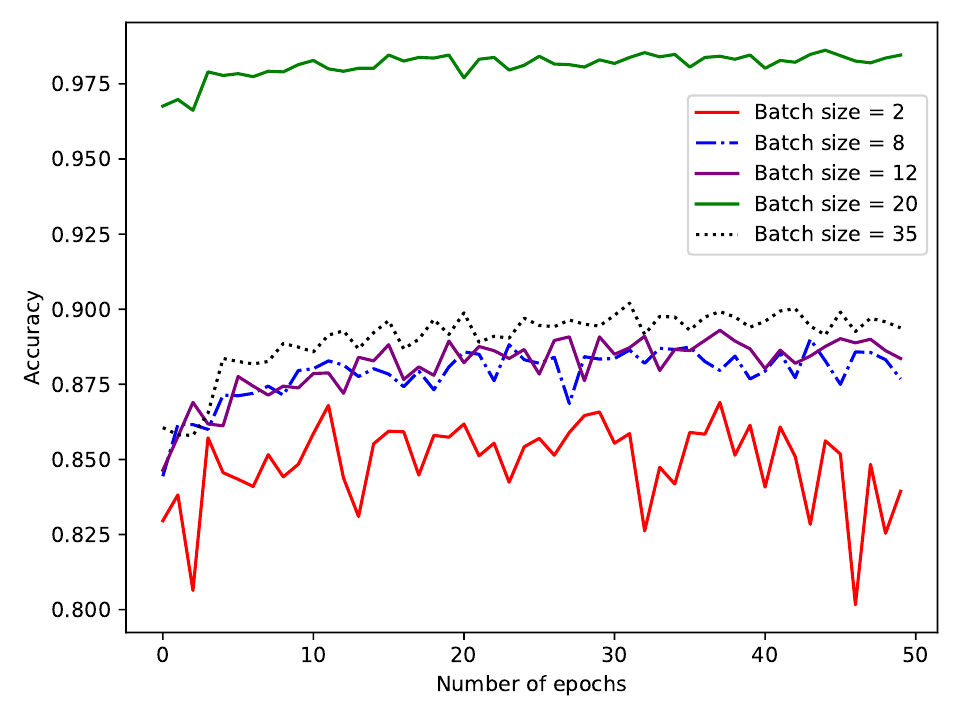}
\caption{\aprox{}: validation accuracy for different batch sizes (learning rate = 0.001).}
\label{fig:diff-bs-approx-accu}
\end{minipage}
\hfill
\begin{minipage}[t]{0.32\linewidth}
\centering
\includegraphics[width=\linewidth]{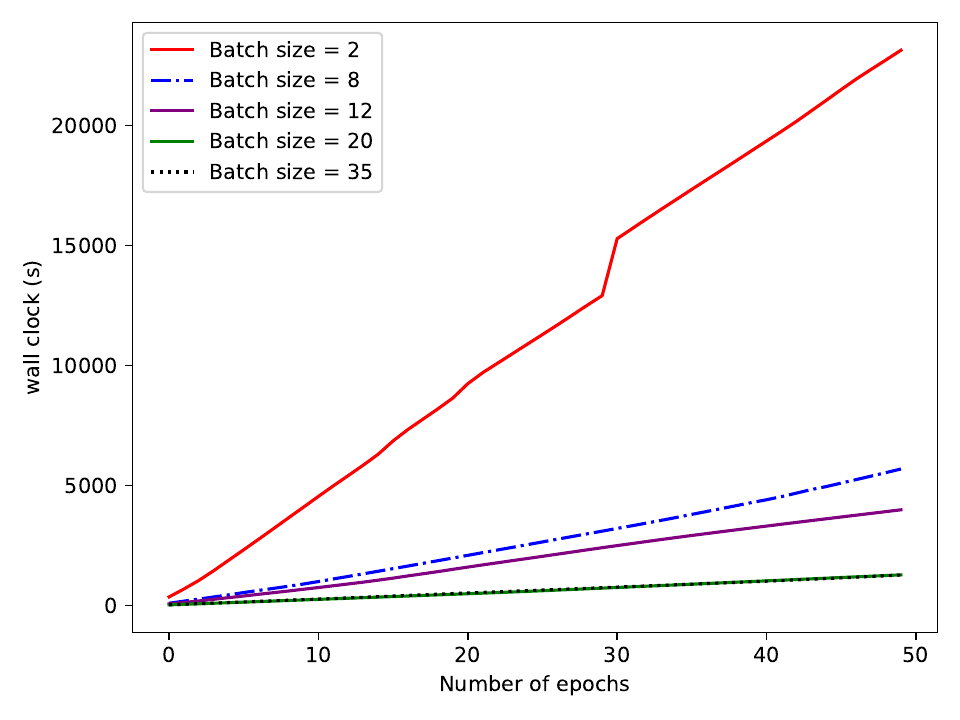}
\caption{\aprox{}: training time for different batch sizes (learning rate = 0.001).}
\label{fig:diff-bs-approx-time}
\end{minipage}
\hfill
\begin{minipage}[t]{0.32\linewidth}
\includegraphics[width=\linewidth]{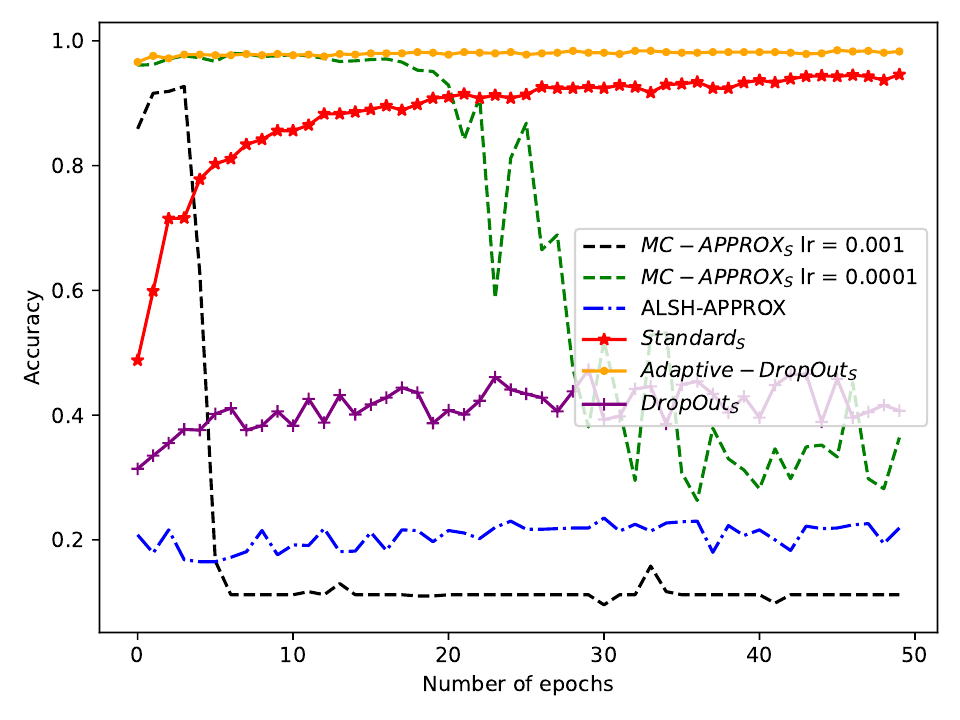}
\caption{Validation accuracies with 7 hidden layers.}
\label{fig:acc-7all}
\end{minipage}
\end{figure*}

\subsection{Hyperparameters}\label{sec:exp:hyper}
\aprox{} is designed for mini-batch stochastic gradient descent, as it uses the set of samples in the mini-batch for error estimation in Equation~\ref{p_i} to identify which rows in $W^k$ to select. In SGD, the estimations would be made using only one sample, and hence are not reliable. As a result, while \aprox{} performs well for mini-batch SGD with a large-enough batch size (20), its efficiency drops significantly for SGD. 
In the stochastic setting, \saprox{} is slower than \sreg{} (Table~\ref{tab:time-mb1}) and is prone to overfitting.

To evaluate this, we run experiments on stochastic setting where batch size is set to 1 (\saprox{}) along with different batch sizes. The results are provided in Figures~\ref{fig:acc-3all},~\ref{fig:diff-bs-approx-accu},~\ref{fig:diff-bs-approx-time},~and~\ref{fig:acc-7all} and Tables~\ref{tab:time-mb1}~and~\ref{tab:time-mb20}. The results in Figure~\ref{fig:diff-bs-approx-accu} show the decrease in accuracy for small mini-batches with the same learning rate: the accuracy drops from 98\% to 84\%. The optimal learning rate to use is smaller for smaller batch sizes~\citep{bs}, so to resolve the overfitting in \saprox{}, we decreased the learning rate from $10^{-3}$ to $10^{-4}$. As shown in Table~\ref{tab:acc} and Figure~\ref{fig:acc-3all}, \saprox{} performs well in terms of accuracy. Moreover, as the model gets more complex by adding hidden layers (Figure~\ref{fig:acc-7all}) and the variance increases, the model is unable to avoid overfitting even with decreasing the learning rate. We discussed in \S\ref{sec:method:fa} that \aprox{} chooses row-column pairs from matrices such that the columns are from the first input matrix $X \in \Reals^{m\times n} $ and the corresponding rows are sampled from the second matrix $W \in \Reals^{n\times n}$. Figure~\ref{fig:acc-7all} provides evidence of the lack of scalability of \saprox{} for deep networks. This can be attributed to the specific sampling procedure employed by \saprox{}. In the stochastic setting, $X_{:,j}$ is reduced to a singleton set. As a result, the time overhead increases, while the reliability of probability estimation for row-column selection decreases.


\subsection{Memory Analysis}\label{mem-analysis}

Our computing architecture is equipped with a hierarchical memory layout that includes a 384 KB Level 1 (L1) cache, 12 MB Level 2 (L2) cache, and a 19.3 MB Level 3 (L3) cache. We find that the \lsh{} algorithm initially requires 24 MB of memory for table setup and subsequently expands by 3,731.9 MB by the end of training. This substantial increase in memory usage indicates that the data that is not cache resident contributes to significant system overload. This plays a crucial role in the execution time of the \lsh{} algorithm, as it necessitates frequent data retrieval from slower, non-cache memory sources.
\aprox{} demonstrates a 45 MB increment in memory usage by the end of training.
Notably, memory usage decreases to 16.4 and 15.7 MB for \dropout{} and \addropout{}, respectively. However, the I/O traffic remains high, and the runtime still increases similarly to \aprox{}. 
This phenomenon is attributable to cache misses during training with \dropout{} and \addropout{}. As \citet{memory-DNN} illustrate, minimizing storage I/O per epoch is more critical than caching data. If data is evicted from the cache before use or is never cached, accessing it from memory increases the I/O overhead. Cache profiling results indicate a roughly 24\% increase in cache misses with \dropout{} and a 27\% increase with \addropout{} compared to \aprox{}, explaining the rise in runtime.

%% file: figs/fig-tab.tex
\begin{figure*}[p]
\centering
\begin{tabular}{@{}c|@{}c@{}|@{}c@{}|@{}c@{}|@{}c@{}}
& \textbf{1 hidden layer} & \textbf{3 hidden layers} & \textbf{5 hidden layers} & \textbf{7 hidden layers} \\
\hline
%
%
%
\rotatebox{90}{\hspace{8.5mm} \bfseries \reg}  
&
\subcaptionbox{\label{fig:reg1}}
{\includegraphics[width=0.225\linewidth]{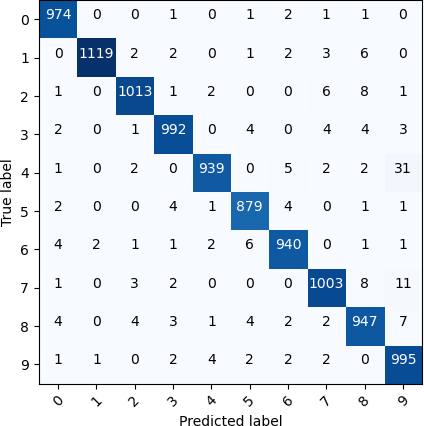}}
& 
\subcaptionbox{\label{fig:reg3}}
{\includegraphics[width=0.225\linewidth]{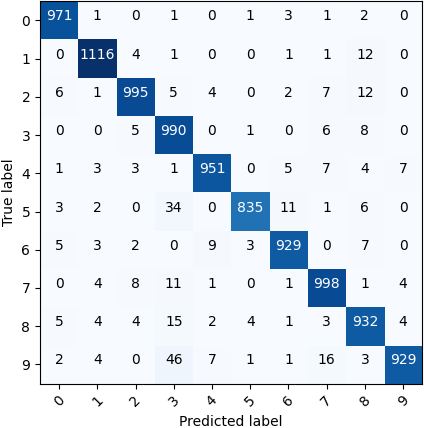}}
& 
\subcaptionbox{\label{fig:reg5}}
{\includegraphics[width=0.225\linewidth]{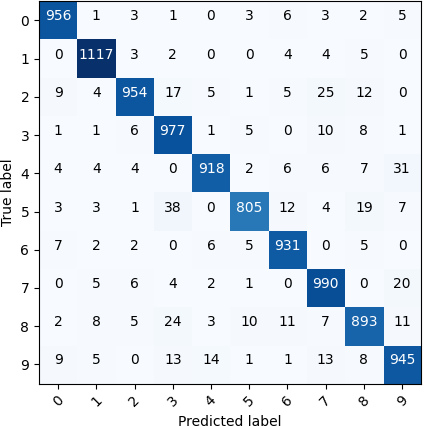}}
& 
\subcaptionbox{\label{fig:reg7}}
{\includegraphics[width=0.225\linewidth]{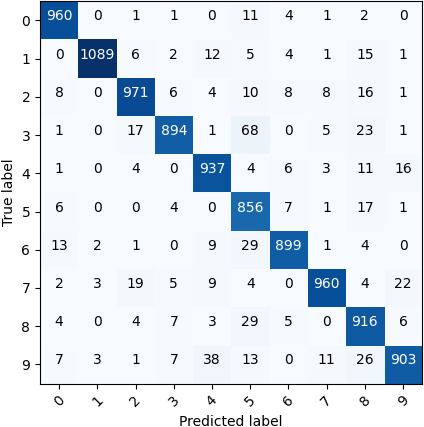}} \\
\hline
%
%
%
\rotatebox{90}{\hspace{9mm} \bfseries \dropout}
&
\subcaptionbox{\label{fig:drop-l1}}
{\includegraphics[width=0.225\linewidth]{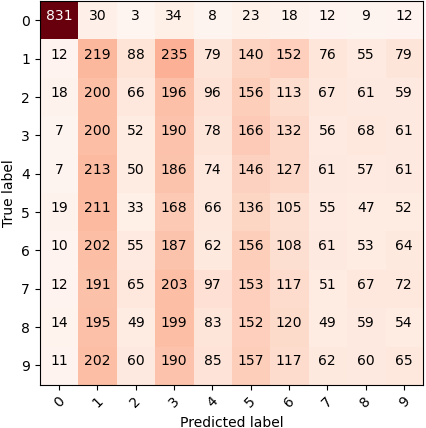}}
&
\subcaptionbox{\label{fig:drop-l3}}
{\includegraphics[width=0.225\linewidth]{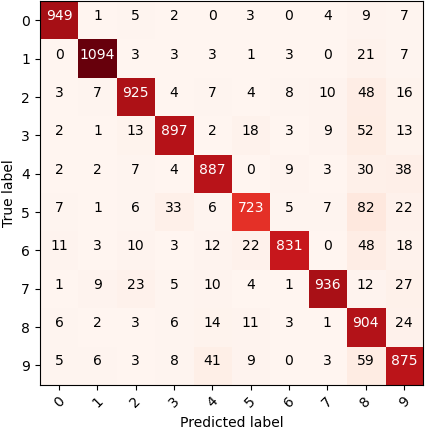}}
& 
\subcaptionbox{\label{fig:drop-l5}}
{\includegraphics[width=0.225\linewidth]{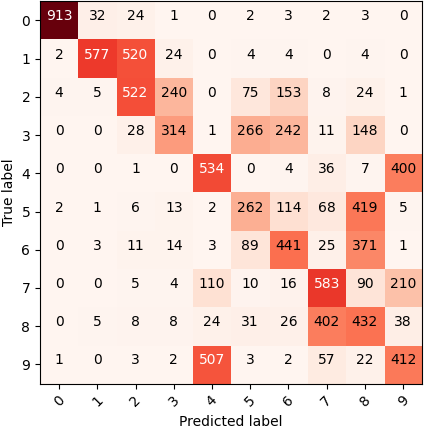}}
& 
\subcaptionbox{\label{fig:drop-l7}}
{\includegraphics[width=0.225\linewidth]{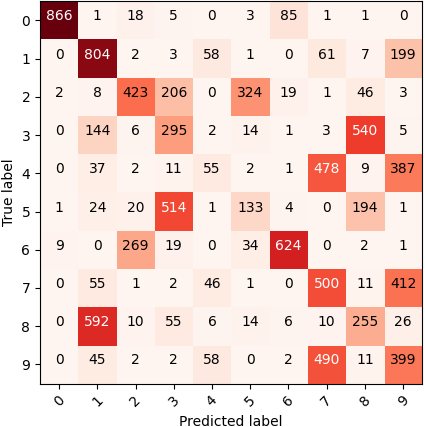}} \\
\hline
%
%
%
\rotatebox{90}{\hspace{0mm} \bfseries \addropout}
&
\subcaptionbox{\label{fig:addrop-l1}}
{\includegraphics[width=0.225\linewidth]{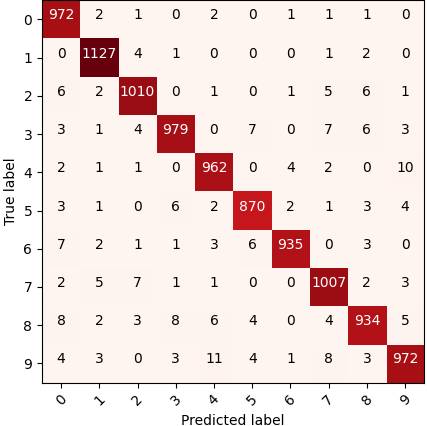}}
& 
\subcaptionbox{\label{fig:addrop-l3}}
{\includegraphics[width=0.225\linewidth]{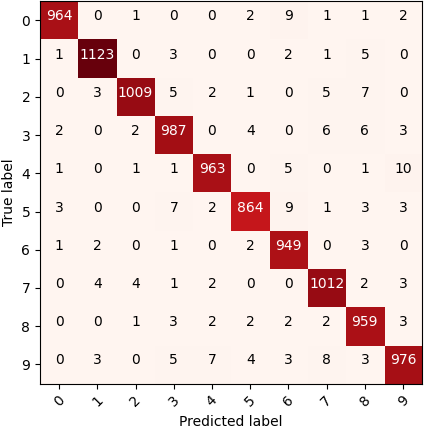}}
& 
\subcaptionbox{\label{fig:addrop-l5}}
{\includegraphics[width=0.225\linewidth]{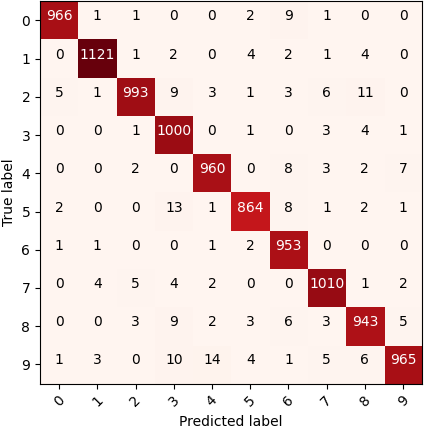}}
& 
\subcaptionbox{\label{fig:addrop-l7}}
{\includegraphics[width=0.225\linewidth]{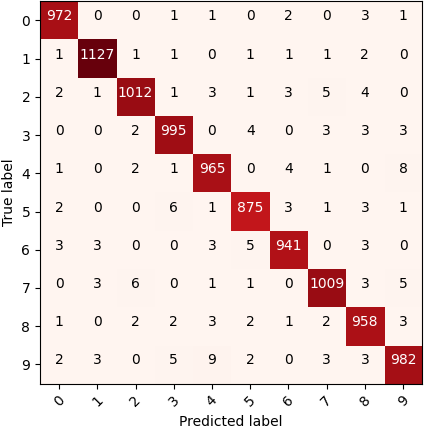}}
\label{fig:addrop-l8} \\
\hline
%
%
%
\rotatebox{90}{\hspace{4mm} \bfseries \lsh}
&
\subcaptionbox{\label{fig:lsh-l1}}
{\includegraphics[width=0.225\linewidth]{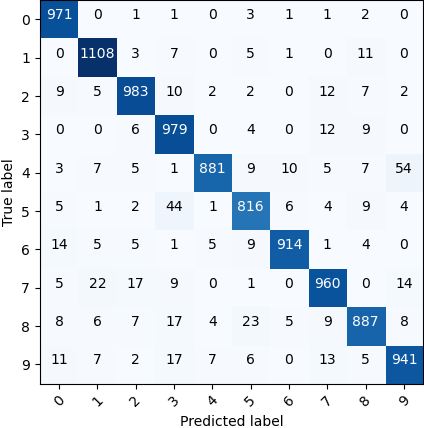}}
& 
\subcaptionbox{\label{fig:lsh-l3}}{\includegraphics[width=0.225\linewidth]{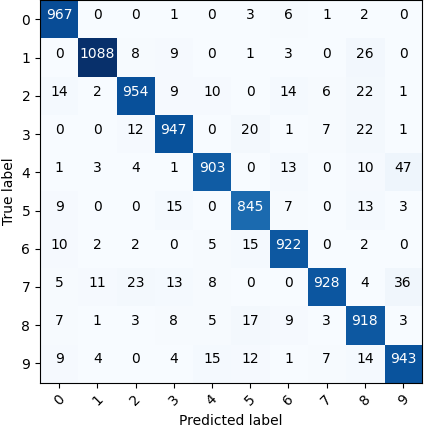}}
& 
\subcaptionbox{\label{fig:lsh-l5}}
{\includegraphics[width=0.225\linewidth]{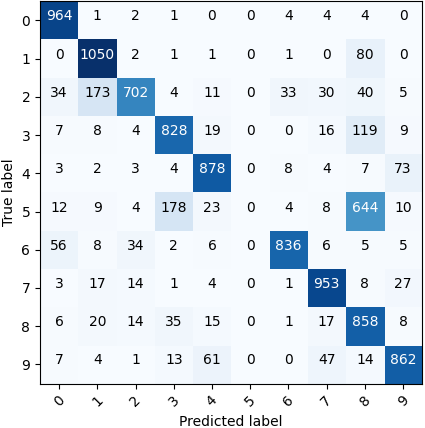}}
& 
\subcaptionbox{\label{fig:lsh-l7}}
{\includegraphics[width=0.225\linewidth]{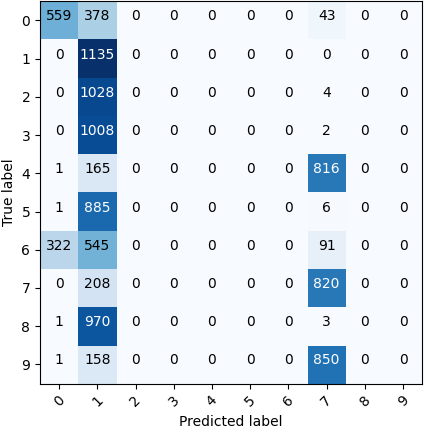}} \\
\hline
%
%
%
\rotatebox{90}{\hspace{7mm} \bfseries \maprox}
&
\subcaptionbox{\label{fig:approx-l1}}
{\includegraphics[width=0.225\linewidth]{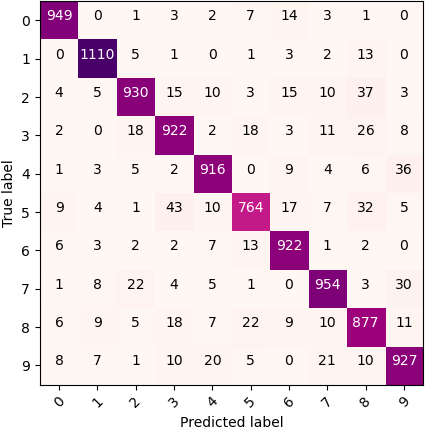}}
& 
\subcaptionbox{\label{fig:approx-l3}}
{\includegraphics[width=0.225\linewidth]{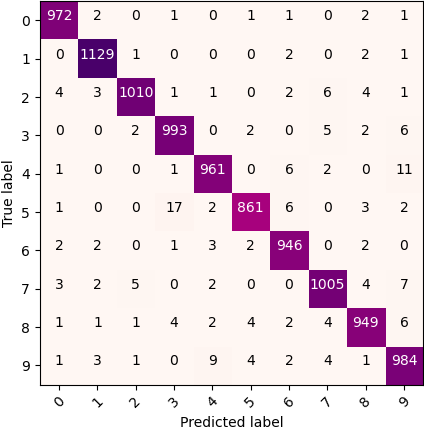}}
& 
\subcaptionbox{\label{fig:approx-l5}}
{\includegraphics[width=0.225\linewidth]{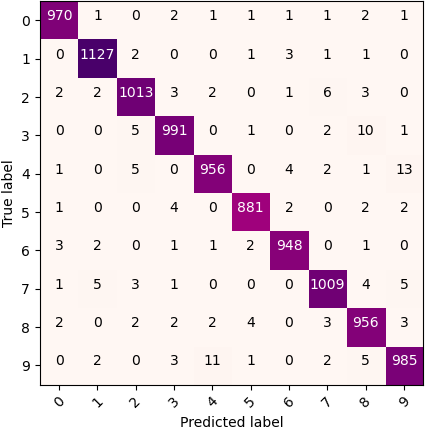}}
& 
\subcaptionbox{\label{fig:approx-l7}}
{\includegraphics[width=0.225\linewidth]{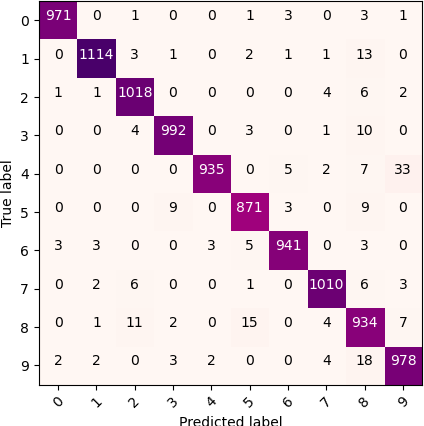}}
\end{tabular}
\caption{Confusion matrices of different algorithms for different numbers of hidden layers. In all plots, x-axis and y-axis are the predicted and true labels (0 to 9), respectively.}
\label{fig:matrices}
\end{figure*}

%% file: disc.tex
\tightsection{Lessons and Discussion}\label{sec:discuss}

It is important to recognize that no single method serves as the optimal solution for optimizing training across all architectures and datasets. 
On the positive side though, our theoretical analysis and experimental evaluations confirm the effectiveness of existing approaches under various settings.
Specifically, our experimental findings discussed in \S\ref{sec:results} confirm the superiority of \aprox{}. When used on MLPs with an appropriately sized mini-batch larger than one, it offers enhanced accuracy, speed and memory efficiency. \aprox{} effectively approximates the inner product at each layer by sampling nodes from the previous layer, and utilizing a batch of data point vectors significantly improves the sampling quality. Another notable observation is that \aprox{} decreases the number of cache misses which significantly contributes to its increased speed.
Moreover, our theoretical analysis in \S~\ref{sec:analysis} demonstrate a small estimation error for \lsh{}, when the number of hidden layers is small. This was observed in our experiment in Figures~\ref{fig:lsh-l1}, Figure~\ref{fig:lsh-l3}, and Table~\ref{tab:acc}. This confirm the effectiveness of \lsh{}, especially with parallel computing~\cite{ScaleableLSH}.

\subsection{Feedforward Approximation Scalability}\label{scaleable-approx}
A major takeaway in this paper is the negative impact of approximation during the feedforward process. First, in \S\ref{sec:analysis}, we theoretically analyzed the error propagation effect from layer to layer. In particular, Theorem~\ref{th:er} highlights the \emph{exponential} increase of gradient estimation error in \lsh{} as the number of hidden layers increases. As a result, for neural networks with more than 3 hidden layers, the error can become even larger than the estimation value. Consequently, the gradient estimation can become utterly arbitrary, resulting in completely inaccurate weight updates during the backpropagation step, which leads to an inaccurate model. For \aprox{}, \citet{FasterNN} did not observe consistent behavior across different models in their experiments. Interestingly, the authors provide the theoretical result that (i) approximating both feedforward and backpropagation operations leads to biased estimates, and (ii) approximating only feedforward operations is unbiased. However, their method for feedforward approximation \emph{failed} in experiments~\citep{FasterNN}. As a result, \aprox{} (the algorithm evaluated in this paper) only adds approximation during backpropagation. We observed similar results for \lsh{} in \S\ref{sec:results}. It is evident in Figures~\ref{fig:matrices},~\ref{fig:acc-lsh-approx},~\ref{fig:acc-7all},~and~\ref{fig:acc-lsh-l} that \lsh{} failed to scale for DNNs, confirming our theoretical analysis in Theorem~\ref{th:er}.

We observed that backpropagation generally took significantly longer than the feedforward step (Tables~\ref{tab:time-mb1}~and~\ref{tab:time-mb20}). Fortunately, backpropagation optimization can significantly improve training time~\citep{DL, bprop}; introducing approximation only during the backpropagation step has the potential to significantly reduce training time. Nevertheless, designing scalable sampling-based algorithms that introduce approximation on both feedforward and backpropagation in DNNs on CPU machines remains an open research direction.


\subsection{DNNs and Small Batch Size}
As observed in our theoretical analysis and experiment results, \lsh{} does not scale to DNNs with more than a few hidden layers. \aprox{}, on the other hand, scales for DNNs, but it is designed based on mini-batch gradient descent and performs well when the batch size is reasonably large (greater than 10). However, the performance of \aprox{} quickly drops for small batch sizes under the same setting. In particular, we observed a swift drop in time efficiency (Figure~\ref{fig:diff-bs-approx-time}) under SGD. While \saprox{} demonstrated a high accuracy in some cases (Table~\ref{tab:acc}), this comes at a cost of a significant increase in training time (even compared to \sreg) and a high risk of overfitting, especially for deep networks (Figure~\ref{fig:acc-7all}). Thus, designing scalable sampling-based algorithms for SGD on CPU remains an open research direction.

\begin{figure}[!bth]
    \includegraphics[width=0.65\linewidth]{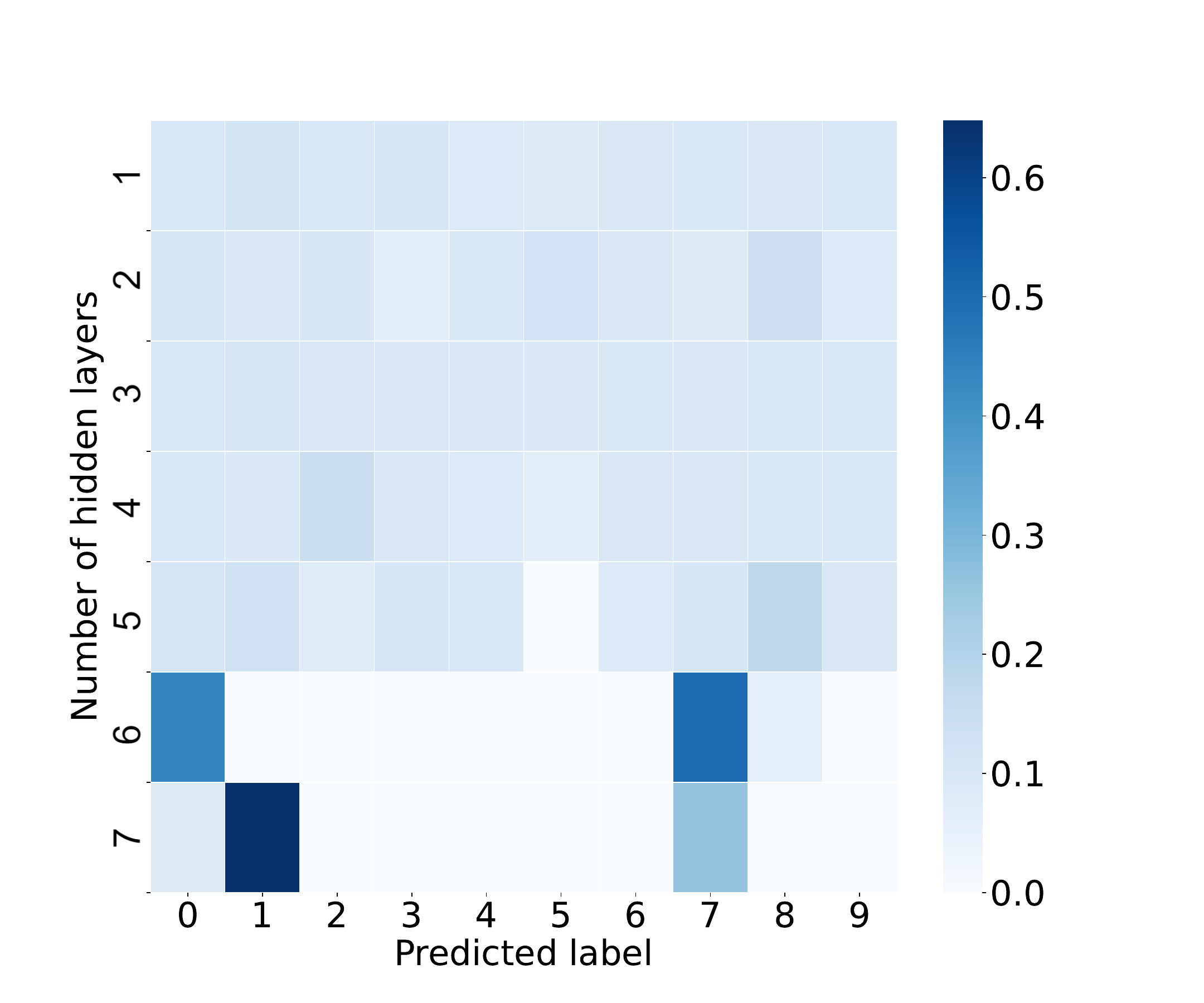}
    \caption{\lsh: impact of network depth on label prediction ratio distribution.}
    \label{fig:labelDist}
\end{figure}

\subsection{\lsh{} Prediction in DNNs}\label{lsh-analysis}
We would like to conclude this section with an interesting observation on \lsh{}. Let us consider the confusion matrices of \lsh{} in Figure~\ref{fig:matrices} (Row~4) once again. From Figure~\ref{fig:lsh-l1}, one can confirm that (i) there is no class imbalance in the test set (approximately same number of samples in each class), and (ii) having high accuracy, the model predictions are uniformly distributed across different classes (approximately same number of samples predicted to be in each class). On the other hand, in Figure~\ref{fig:lsh-l7}, it seems not only the model is inaccurate, but interestingly only a few labels from the class labels are generated in predictions (all samples are labeled as either 0, 1, or 7). We repeated the experiment multiple times and observed consistent behavior. Furthermore, comparing Row~2 (\dropout{}) with Row~4 (\lsh{}) in Figure~\ref{fig:matrices}, we note that, while both methods failed to scale with the number of hidden layers, \dropout{} maintains the label diversity in its prediction, which demonstrates randomness. To better present this, in Figure~\ref{fig:labelDist}, we provide the ratio of test-set samples predicted for each label (column) for networks with various numbers of hidden layers. It is clear that, while initially the label prediction distribution is uniform, as the number of layers increases, the predictions concentrate around a few arbitrary classes. The reason is that, while training the model using \lsh{}, as the gradient estimation error increases for deeper networks, a small subset of nodes remains active in deeper layers, regardless of the input sample. As a result, the set of edges for which the weights get updated remains almost the same. Therefore, when predicting the label of an input sample, the same set of nodes is ``activated'', resulting in a small set of predictions generated.

\subsection{Optimal Choice of Training Method}
 Building on the insights and lessons from our study, we present a decision tree to guide users in selecting the most effective method for training DNNs on CPU machines:

\begin{center}
\small
\vspace{-2.5mm}\begin{tikzpicture}[level distance=1.1cm, level 1/.style={sibling distance=2.8cm}, level 2/.style={sibling distance=2.8cm}]
\node {Setting}
    child {node {\# Layers}
        child {node {Parallel Computing}
            child {node {\lsh{} (\cite{ScaleableLSH})} edge from parent node [left, draw=none] {yes}}
            child {node {{\bfseries Open Problem} (Tab.~\ref{tab:time-mb1})} edge from parent node [right, draw=none] {no}} edge from parent node [left, draw=none] {Shallow ($\leq 4$)}}
        child {node {{\bfseries Open Problem}\footnotemark{} (Fig.~\ref{fig:matrices})} edge from parent node [right, draw=none] {Deep ($> 4$)}} edge from parent node [left, draw=none] {SGD}}
    child {node {\aprox{} (\S\ref{sec:exp:hyper}, Tab.~\ref{tab:time-mb20})} edge from parent node [right, draw=none] {Mini-Batch SGD}};
\end{tikzpicture}
\end{center}

Our primary objective is to identify methods that deliver performance nearly comparable to standard training in terms of accuracy while ensuring faster execution on our CPU system across various settings. Our results, presented in Section \ref{sec:exp:hyper} and Table \ref{tab:time-mb20}, indicate that \aprox{} surpasses other methods in mini-batch settings. In stochastic environments, the effective method varies with the network depth. \footnotetext{``Open problem'' refers to settings where existing sampling-based algorithms failed; further research is needed to design algorithms for those settings.}
The experimental evaluation by
\citet{ScaleableLSH} confirms that \lsh{} scales effectively using parallel computing with multi-processors up to $2^6$ processors, for up to four layers.
Therefore, we believe \lsh{} performs optimally up to four layers under such conditions.

\tightsection{Conclusion}\label{sec:conclusion}
Many of the advanced technologies originally developed for addressing big data challenges have been extended to solve scalability complexities across various domains.
In this paper, we evaluated one of these settings, where sampling-based techniques were proposed for training DNN on CPU machines with limited resources.
To this end, we made connection between two sampling-based research directions that can be viewed as matrix multiplication approximations.
We provided theoretical analyses, followed by extensive empirical evaluations.

Our results demonstrate a correlation between the number of hidden layers and approximation error in DNNs under hashing-based methods. In addition, we provided insights into the performance of fast training methods in different settings and highlight areas for further research.

As a final note, energy consumption during DNN training raises environmental concerns. Recent studies have explored the significant carbon footprint associated with large-scale neural networks, primarily due to their energy consumption~\citep{carbon-fprint, green-AI}. One interesting direction for future work is to study the impact of sampling-based techniques on energy efficiency.